\documentclass{article} % Anonymized submission

\usepackage[final]{neurips_2021}

\usepackage[english]{babel}
\usepackage{todonotes}

\usepackage{xcolor}
[preprint]

\usepackage{cancel}
\usepackage{calc}
\usepackage{mathtools}
\usepackage{empheq}
\usepackage{amsfonts}
\usepackage{amssymb}
\usepackage{amsmath}
\usepackage{amsthm}
\usepackage{thmtools}

\usepackage{booktabs}

\usepackage{thm-restate}
\newtheorem{definition}{Definition}
\declaretheorem[name=Theorem]{theorem}
\newtheorem{lemma}{Lemma}
\newtheorem{corollary}{Corollary}

\newcommand{\op}{\operatorname*}

\usepackage{nameref}
\definecolor{blueish}{RGB}{103, 135, 176}
\definecolor{reddish}{RGB}{ 205, 102, 7}
\usepackage{hyperref}
\hypersetup{
	colorlinks,
	linkcolor={reddish!70!black},
	citecolor={blueish!70!black},
	urlcolor=black,
}
\usepackage[nameinlink,noabbrev]{cleveref}
\crefname{equation}{}{}
\Crefname{equation}{}{}
\Crefname{definition}{Definition}{Definitions}
\Crefname{lemma}{Lemma}{Lemmas}
\crefname{lemma}{Lemma}{Lemmas}
\Crefname{corollary}{Corollary}{Corollaries}
\crefname{corollary}{Corollary}{Corollaries}

\usepackage{algorithm}
\usepackage{algpseudocode}
\usepackage{listings}
\usepackage{color}
\definecolor{codegreen}{rgb}{0,0.6,0}
\definecolor{codegray}{rgb}{0.5,0.5,0.5}
\definecolor{codepurple}{rgb}{0.58,0,0.82}
\definecolor{backcolour}{rgb}{0.95,0.95,0.92}

\lstdefinestyle{mystyle}{
	backgroundcolor=\color{backcolour},   
	commentstyle=\color{codegreen},
	keywordstyle=\color{magenta},
	numberstyle=\tiny\color{codegray},
	stringstyle=\color{codepurple},
	basicstyle=\footnotesize,
	breakatwhitespace=false,         
	breaklines=true,                 
	captionpos=b,                    
	keepspaces=true,                 
	numbers=left,                    
	numbersep=5pt,                  
	showspaces=false,                
	showstringspaces=false,
	showtabs=false,                  
	tabsize=2
}

\usepackage{tikz}
\usetikzlibrary{calc}
\usetikzlibrary{mindmap,trees}
\usetikzlibrary{patterns}
\usetikzlibrary{bayesnet}

\makeatletter
\newcommand\appendToPath[1]
{
	\providecommand*{\Ginput@path}{}
	\g@addto@macro\Ginput@path{{#1}}
	
	\providecommand*{\input@path}{}
	\g@addto@macro\input@path{{#1}}% append
}
\makeatother

\usepackage{subcaption}

\title{Regret Bounds for Gaussian-Process Optimization in Large Domains}

\author{%
	Manuel W\"uthrich\thanks{\texttt{manuel.wuthrich@pm.me}} \\
	MPI for Intelligent Systems\\
	T\"ubingen, Germany \\
	\And
	Bernhard Sch\"olkopf\\
	MPI for Intelligent Systems\\
	T\"ubingen, Germany \\
	\And 
	Andreas Krause\\
	ETH Zurich\\
	Zurich, Switzerland\\
}

\begin{document}

	\global\long\def\ref{\Cref}

\maketitle

\begin{abstract}%
The goal of this paper is to characterize Gaussian-Process optimization in the setting where the function domain is large relative to the number of admissible function evaluations, i.e., where it is impossible to find the global optimum. We provide upper bounds on the suboptimality (Bayesian simple regret) of the solution found by optimization strategies that are closely related to the widely used expected improvement (EI) and upper confidence bound (UCB) algorithms. These regret bounds illuminate the relationship between the number of evaluations, the domain size (i.e. cardinality of finite domains / Lipschitz constant of the covariance function in continuous domains), and the optimality of the retrieved function value. In particular, we show that even when the number of evaluations is far too small to find the global optimum, we can find nontrivial function values (e.g. values that achieve a certain ratio with the optimal value).

\end{abstract}

\section{Introduction}

\label{sec:introduction}

In practice, nonconvex, gradient-free optimization problems arise frequently, for instance when tuning the parameters
of an algorithm or optimizing the controller of a physical system (see e.g. \cite{shashariari_16_bayes_opt_survey})
for more concrete examples). Different versions of this problem have been
addressed e.g. in the multi-armed-bandit and the Bayesian-optimization literature (which we review below).
However, the situation where the number 
of admissible function evaluations is too small to identify the 
global optimum has so far received little attention, despite arising frequently in practice. For instance, when optimizing the hyper-parameters of a machine-learning method or the controller-parameters of a physical system, it is typically impossible to explore the domain exhaustively.
We take first steps towards a better understanding of this setting through a theoretical analysis based on the adaptive-submodularity
framework by~\cite{golovin11adaptive}. It is not the purpose of the present paper to propose a novel algorithm with better performance (we modify the EI and UCB strategies merely to facilitate the proof), but rather to gain an understanding of how GP-optimization performs in the aforementioned setting, as a function of the number of evaluations and the domain size.

As is done typically, we model the uncertainty about the underlying function as a Gaussian Process (GP). The question is how close we can get
to the global optimum with $T$ function evaluations, where $T$ is
small relative to the domain of the function (i.e., it is impossible to identify the global optimum with high
confidence with only $T$ evaluations). 
As a performance measure, we use the Bayesian simple regret, i.e., the
expected difference between the optimal function value and the value
attained by the algorithm. For the discrete case with domain size $N$, we derive a problem-independent regret
bound (i.e., it only depends on $T,N$ and is worst-case in terms
of the GP prior) for two optimization algorithms that are closely
related to the well-known expected improvement (EI, \cite{jones_1998})
and the upper confidence bound (UCB, \cite{auer_2002,gpucb}) methods.
In contrast to related work, our bounds are non-vacuous even when
we can only explore a small fraction of the function. We extend this result to continuous domains and show that the resulting bound scales better with the size of the function domain (equivalently, the Lipschitz constant of the covariance function) than related work (\cite{Grunewalder2010-uu}).

\subsection{Related Work}
The multi-armed bandit literature is closest to the present paper.
In the multi-armed bandit problem, the agent faces a row of $N$ slot
machines (also called one-armed bandits, hence the name). The agent
can decide at each of $T$ rounds which lever (arm) to pull, with
the objective of maximizing the payoff. This problem was originally
proposed in \cite{robbins_1952} to study the trade-off between exploration
and exploitation in a generic and principled way. Since then, this
problem has been studied extensively and important theoretical results
have been established.

Each of the $N$ arms has an associated reward distribution, with
a fixed but unknown mean $F_{n}$. When pulling arm $n$ at round
$t$, we obtain a payoff $Y_{t}:=F_{n}+E_{n,t}$ where the noise is
zero mean $\mathbb{E}[E_{n,t}]=0$, independent across time steps
and arms, and identically distributed for a given arm across time
steps. Let us denote the arm pulled at time $t$ by $A_{t}$. Performance
is typically analyzed in terms of the regret, which is defined as
the difference between the mean payoff of the optimal arm and the
one pulled at time $t$ 
\begin{align}
	R_{t}:=\max_{n\in[N]}F_{n}-F_{A_{t}}.
\end{align}
Here, we are interested in the setting where $E_{n,t}$ is small or
even zero and we are allowed to pull only few arms $T<N$. We classify
related work according to the information about the mean $F_{n}$ and
the noise $E_{n,t}$ that is available to the agent.

\subsubsection{No Prior Information}

Traditionally, most authors considered the case where only very
basic information about the payoff $F_{n}+E_{n,t}$ is available to the
agent, e.g. that its distribution is a member of a given family of distributions
or that it takes values in an interval, typically $[0,1]$.

In a seminal paper, \citet{lai_1985} showed that the cumulative regret
$\sum_{t\in[T]}R_{t}$ grows at least logarithmically in $T$ and
proposed a policy that achieves this rate asymptotically. Later, \citet{auer_2002}
proposed an upper confidence bound (UCB) strategy and proved that
it achieves logarithmic cumulative regret in $T$ uniformly in time,
not just asymptotically. Since then, many related results
have been obtained for UCB and other strategies, such as Thompson sampling
\cite{agrawal_2012,kaufmann_2012}.
A number of similar results have also been obtained for the objective
of best arm identification \cite{madani_2004,bubeck_2009,audibert_2010},
where we do not care about the cumulative regret, but only about the
lowest regret attained.

However, all of these bounds are nontrivial only when the number of
plays is larger than the number of arms $T>N$. This is not surprising,
since no algorithm can be guaranteed to perform well with a low number
of plays with such limited information. To see this, consider an example
where $F_{i}=1$ and $F_{n}=0~\forall n\neq i$. Since the agent has
no prior information about which is the right arm $i$, the best it
can do is to randomly try out one after the other. Hence it is clear
that to obtain meaningful bounds for $T<N$ we need more prior knowledge
about the reward distributions of the arms.
\subsubsection{Structural Prior Information}
In \cite{dani_2008} the authors consider the problem where the mean
rewards at each bandit are a linear function of an unknown, $K$-dimensional
vector. However, similarly to the work above, the difficulty addressed
in this paper is mainly the noise, i.e. the fact that the same arm
does not always yield the same reward. For $E_{n,t}=0$, the $\mathcal{O}(K)$
cumulative regret bounds derived in this paper are trivial, since
with $\mathcal{O}(K)$ evaluations we can simply identify the linear function.

\subsubsection{Gaussian Prior}

More closely related to our work, a number of authors have considered
Gaussian priors. \citet{bull_2011} for instance provides asymptotic
convergence rates for the EI strategy with Gaussian-process priors
on the unknown function.

\citet{gpucb} propose a UCB algorithm for the Gaussian-process setting.
The authors give finite-time bounds on the cumulative regret.
However, these bounds are intended for a different setting than the one
we consider here: 1) They allow for noisy observations and 2) they 
are only meaningful if we are allowed to make a sufficient number of 
evaluations to attain low uncertainty over the entire GP. Please see \ref{sec:gpucb} for more details on the second point.

\citet{Russo2014-pj} use a similar
analysis to derive bounds for Thompson sampling, which are therefore
subject to similar limitations. In contrast, the bounds in 
\citet{Russo2016-xz} do not depend on the entropy of the entire GP,
but rather on the entropy of the optimal action. However,
our goal here is to derive bounds that are meaningful even when 
the optimal action cannot be found, i.e. its entropy remains large.

\citet{freitas_2012} complement the work in \citep{gpucb} by providing regret
bounds for the setting where the function can be observed without
noise. However, these are asymptotic and therefore not applicable
to the setting we have in mind. 

Similarly to the present paper, \citet{Grunewalder2010-uu} analyze
the Bayesian simple regret of GP optimization. They provide a lower and
an upper bound on the regret of the optimal policy for GPs on continuous domains with covariance
functions that satisfy a continuity assumption.
Here, we build on this work and derive a bound with an improved 
dependence on the Lipschitz constant of the covariance function, i.e., our bound 
scales better with decreasing length-scales (and, equivalently, larger domains).
Unlike \citep{Grunewalder2010-uu}, we also consider GPs on finite domains without
any restrictions on the covariance.

\subsubsection{Adaptive Submodularity}
The adaptive-submodularity framework of \cite{golovin11adaptive} is
in principle well suited for the kind of analysis we are interested
in. However, we will show that the problem at hand is not adaptively
submodular, but our proof is inspired by that framework.

\section{Problem Definition}

\subsection{The Finite Case}\label{sec:po_problem_statement} 
In this section we specify the type
of bandit problem we are interested in more formally. The goal is
to learn about a function with domain $\mathcal{A}=[N]$ (we will use $[N]$ to denote the set $\{1,...,N\})$ and co-domain $\mathbb{R}$.
We represent the function as a sequence $F=(F_{n})_{n\in[N]}$. Our
prior belief about the function $F$ is assumed to be Gaussian 
\begin{align}
	F\sim & \mathcal{N}(\mu,\Sigma).\label{eq:ps_prior}
\end{align}
At each of the $T$ iterations, we pick an action (arm) $A_{t}$ from
$[N]$ at which we evaluate the function. After each action, an observation
$Y_{t}\in\mathbb{R}$ is returned to the agent 
\begin{align}
	Y_{t}:=F_{A_{t}}.
\end{align}
Note that here we restrict ourselves to the case where the function
can be evaluated without noise.

For convenience, we introduce some additional random variables, based
on which the optimization algorithm will pick where to evaluate the
function next. We denote the posterior mean and covariance at time
$t$ by 
\begin{align}
	M_{t}:&=\mathbb{E}[F|A_{:t},Y_{:t}]\\
	C_{t}: & =\mathbb{COV}[F|A_{:t},Y_{:t}].
\end{align}
In addition, we will need the maximum and minimum observations up to
time $t$ 
\begin{align}
	\hat{Y}_{t}: & =\max_{k\in[t]}Y_{k}\quad\forall t\in\{1,..,T\}\\
	\check{Y}_{t}: & =\min_{k\in[t]}Y_{k}\quad\forall t\in\{1,..,T\}.
\end{align}
Furthermore, we will make statements about the difference between
the smallest and the largest observed value
\begin{align}
	\hat{\check{Y}}_{t}: & =\hat{Y}_{t}-\check{Y}_{t}\quad\forall t\in\{1,..,T\}.
\end{align}
Finally, for notational convenience we define 
\begin{align}
	\hat{Y},\check{Y},\hat{\check{Y}}: & =\hat{Y}_{T},\check{Y}_{T},\hat{\check{Y}}_{T}.
\end{align}
Analogously, let us define the function minimum $\check{F}:=\min_{n\in[N]}F_{n}$, maximum $\hat{F}: =\max_{n\in[N]}F_{n}$ and difference $\hat{\check{F}}:  =\hat{F}-\check{F}$.

\subsubsection{Problem Instances}
A problem instance is defined by the tuple $(N,T,\mu,\Sigma)$, i.e.
the domain size $N\in\mathbb{N}_{>0}$, the number of rounds $T\in\mathbb{N}_{>0}$
and the prior \ref{eq:ps_prior} with mean $\mu\in\mathbb{R}^N$
and covariance $\Sigma\in\mathbb{S}_{+}^{n}$, where we use $\mathbb{S}_{+}^{n}$ to denote the set of positive semidefinite matrices of size $n$.

\subsection{The Continuous Case}
The definitions in the continuous case are analogous. A problem instance here is defined by $(\mathcal{A}, T, \mu, k)$, i.e. the function domain $\mathcal{A}$, the number of rounds $T\in\mathbb{N}_{>0}$, a mean function 
 $\mu:\mathcal{A}\to \mathbb{R}$ and a positive semi-definite kernel $k:\mathcal{A}^2 \to \mathbb{R}$.

\section{Results}
In this section we provide bounds on the Bayesian simple regret 
that hold for the two different optimization algorithms we describe
in the following.

\subsection{Optimization Algorithms}
  Two of the most widely used GP-optimization algorithms are the
 expected improvement (EI) (\cite{jones_1998}) and the upper confidence bound (UCB) (\cite{auer_2002,gpucb}) methods.
 In the following, we define two optimization policies that are closely related:
 \begin{definition}\label{def:ei} The
	expected improvement \cite{bull_2011} is defined as 
	\begin{align}
	\op{ei}(\tau): & =\int_{-\infty}^{\infty}\max\{x-\tau,0\}\mathcal{N}(x)dx =\mathcal{N}(\tau)-\tau\Phi^{c}(\tau)
	\end{align}
	where $\mathcal{N}$ is the standard normal density function and
	$\Phi^{c}$ is the complementary cumulative density function of a
	standard normal distribution. Furthermore, we use the notation 
	\begin{align}
	\op{ei}(\tau|\mu,\sigma) & =\sigma\op{ei}\left(\frac{\tau-\mu}{\sigma}\right)\label{eq:ei_normalization} =\int_{-\infty}^{\infty}\max\{x-\tau,0\}\mathcal{N}(x|\mu,\sigma)dx.
	\end{align}
\end{definition}

\begin{definition}[EI2]\label{def:ei_extremization} An
	agent follows the EI2 strategy when it picks its actions according
	to 
	\begin{align}
		A_{t+1} & =\op{argmax}_{n\in[N]}\max\left\{ \op{ei}\left(\hat{Y}_{t}|M_{t}^{n},\sqrt{C_{t}^{nn}}\right),\op{ei}\left(-\check{Y}_{t}|-M_{t}^{n},\sqrt{C_{t}^{nn}}\right)\right\} 
	\end{align} %
	with the expected improvement $\op{ei}$ as defined in \ref{def:ei}.
\end{definition}

\begin{definition}[UCB2]\label{def:ucb_extremization} An agent follows the UCB2 strategy when it picks its actions according
	to 
	\begin{align}
		A_{t+1} & =\op{argmax}_{n\in[N]}\max\left\{ -\hat{Y}_{t}+M_{t}^{n}+\sqrt{C_{t}^{nn}2\log N},\check{Y}_{t}-M_{t}^{n}+\sqrt{C_{t}^{nn}2\log N}\right\} .
	\end{align}
\end{definition} 
The main difference to the standard versions of these methods (see \ref{def:ei_extremization_standard} and \ref{def:ucb_extremization_standard}) is that the algorithms here are symmetrical in
the sense that they are invariant to flipping the sign of the GP.
They maximize \emph{and} minimize at the same time by picking the
point which we expect to either increase the observed maximum \emph{or}
decrease the observed minimum the most. This symmetry is important
for our proof, whether the same bound also holds for following the
standard, one-sided EI or UCB strategies is an open question, we discuss this point in detail in \ref{sec:relation_to_standard}.

\subsection{Upper Bound on the Bayesian Simple Regret for the Extremization Problem}
Here, we provide a regret bound for function extremization (i.e. the goal
is to find both the minimum and the maximum) from which the regret bound
for function maximization will follow straightforwardly. Note that 
the bound is problem-independent in the sense that it does not depend
on the prior $(\mu,\Sigma)$.

\begin{restatable}{theorem}{maintheorem}\label{theorem:main}
	For any instance
	$(N,T,\mu,\Sigma)$ of the problem defined in \ref{sec:po_problem_statement}
	with $N\ge T\ge500$, if we follow either the EI2 (\ref{def:ei_extremization})
	or the UCB2 (\ref{def:ucb_extremization}) strategy, we have 
	\begin{equation}
	\frac{\mathbb{E}\left[\hat{\check{F}}\right]-\mathbb{E}\left[\hat{\check{Y}}\right]}{\mathbb{E}\left[\hat{\check{F}}\right]}\le1-\left(1-T^{-\frac{1}{2\sqrt{\pi}}}\right)\sqrt{\frac{\log(T)-\log\left(3\log^{\frac{3}{2}}(T)\right)}{\log(N)}}.
	\end{equation}
	This guarantees that the expected difference between the maximum and
	minimum retrieved function value achieves a certain ratio with respect
	to the expected difference between the global maximum and the global
	minimum.
\end{restatable}

\begin{proof} The full proof can be found
	in the supplementary material in \ref{sec:proof}, here we only give an outline. The proof is inspired
	by the adaptive submodularity framework by \cite{golovin11adaptive}.
	The problem at hand can be understood as finding the policy (optimization
	algorithm) which maximizes an expected utility. 
	Finding this optimal policy would require solving a partially observable
	Markov decision process (POMDP), which is intractable in most relevant
	situations. Instead, a common approach is to use a greedy policy which
	maximizes the expected single-step increase
	in utility at each time step $t$, which is in our case 
	\begin{align}
		A_{t+1}=\op{argmax}_{a}\left(\mathbb{E}\left[\hat{Y}_{t+1}-\check{Y}_{t+1}|A_{t+1}=a,A_{:t},Y_{:t}\right]-\mathbb{E}\left[\hat{Y}_{t}-\check{Y}_{t}|A_{:t},Y_{:t}\right]\right).
	\end{align}
	This corresponds to the EI2 (\ref{def:ei_extremization}) strategy.
	The task now is to show that the greedy policy will not perform much
	worse than the optimal policy. \citet{golovin11adaptive} show that
	this holds for problems that are adaptively submodular (among some
	other conditions). In the present problem, we can roughly translate
	this condition to: The progress we make at a given time step has to
	be proportional to how far the current best value is from the global
	optimum. While our problem is not adaptively submodular (see \ref{app:as}), we show that a similar condition holds (which leads to similar guarantees).
\end{proof}

\subsection{Upper and Lower Bound on the Bayesian Simple Regret for the Maximization Problem}

For maximization, the goal is to minimize the Bayesian simple regret,
i.e. the expected difference between the globally maximal function
value and the best value found by the optimization algorithm 
\begin{equation}
\mathbb{E}[\hat{F}]-\mathbb{E}[\hat{Y}].
\end{equation}
As is often done in the literature (see e.g. \cite{gpucb}), we restrict ourselves here to centered GPs (i.e. zero prior mean; naturally, the mean will change during the optimization). 
To be invariant to scaling of the prior distribution, we normalize the regret with the expected global maximum:
\begin{equation}
\text{normreg}:=\frac{\mathbb{E}\left[\hat{F}\right]-\mathbb{E}\left[\hat{Y}\right]}{\mathbb{E}\left[\hat{F}\right]}.
\end{equation}

\subsubsection{Upper Regret Bound for the EI2 and UCB2 Policies}
We obtain the following upper bound on this normalized Bayesian simple regret:

\begin{corollary}\label{coro:maximization} For any instance $(N,T,\mu,\Sigma)$
	of the problem defined in \ref{sec:po_problem_statement} with zero
	mean $\mu_{n}=0~\forall n$ and $N\ge T\ge500$, if we follow
	either the EI2 (\ref{def:ei_extremization}) or the UCB2 (\ref{def:ucb_extremization})
	strategy, we have 
	\begin{equation}
	\op{normreg}\le1-\left(1-T^{-\frac{1}{2\sqrt{\pi}}}\right)\sqrt{\frac{\log(T)-\log\left(3\log^{\frac{3}{2}}(T)\right)}{\log(N)}}.
	\end{equation}
	This implies a bound on the ratio between the expected maximum found
	by the algorithm and the expected global maximum.
	
\end{corollary}\begin{proof} The Gaussian prior
	is symmetric about $0$, as the mean is $0$ everywhere (i.e. the probability density satisfies $p(F=f)=p(F=-f)~\forall f$).
	Since also both the EI2 and the UCB2 policies are symmetric, we have $\mathbb{E}\left[\hat{F}\right]=-\mathbb{E}\left[\check{F}\right]$ and hence 
	\begin{equation}
	\mathbb{E}\left[\hat{\check{F}}\right]=\mathbb{E}\left[\hat{F}\right]-\mathbb{E}\left[\check{F}\right]=2\mathbb{E}\left[\hat{F}\right]
	\end{equation}
	and similarly $\mathbb{E}\left[\hat{\check{Y}}\right]=2\mathbb{E}\left[\hat{Y}\right]$.
	Substituting this in \ref{theorem:main}, \ref{coro:maximization}
	follows.\end{proof}

In \ref{fig:regret_plot}, we plot this bound, and we observe that
we obtain nontrivial regret bounds even when evaluating the function
only at a small fraction of its domain. 
\begin{figure}
	\centering
	\begin{minipage}{0.48\textwidth}
		\centering \includegraphics[width=1.0\linewidth]{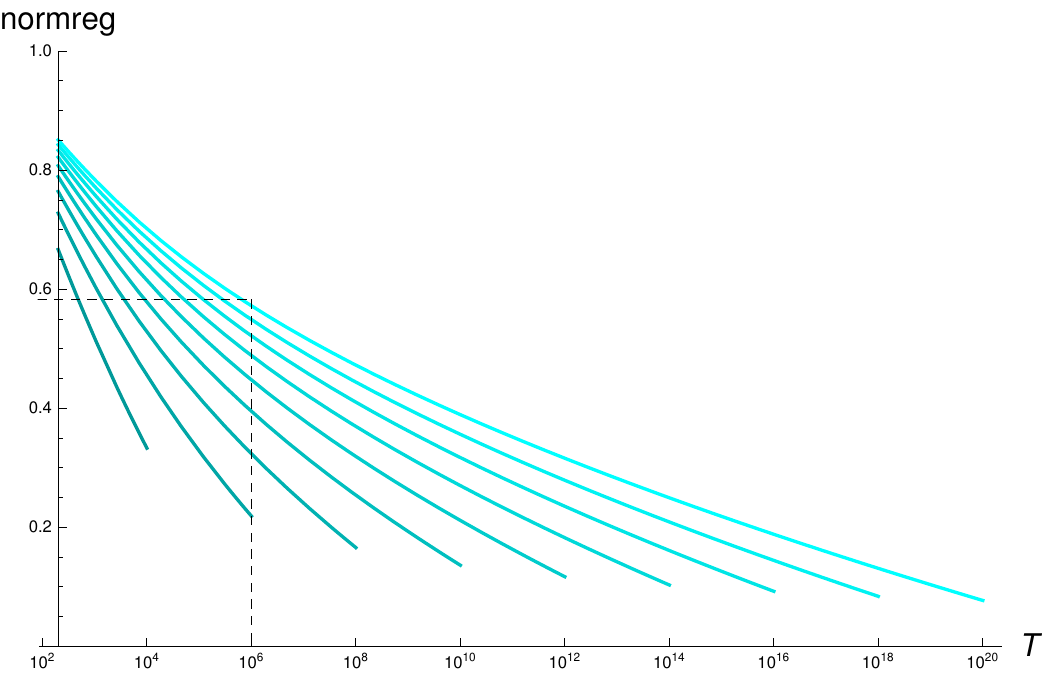}
		\caption{The upper bound on the regret from \ref{coro:maximization} as a function
			of the number of evaluations $T$. Each curve corresponds to a different
			domain size $N$ and is plotted for $T\in\{500,...,N\}$, hence $N$
			can be read from the end point of each curve.}
		\label{fig:regret_plot} 
	\end{minipage}\hfill
	\begin{minipage}{0.48\textwidth}
		\centering \includegraphics[width=1.0\linewidth]{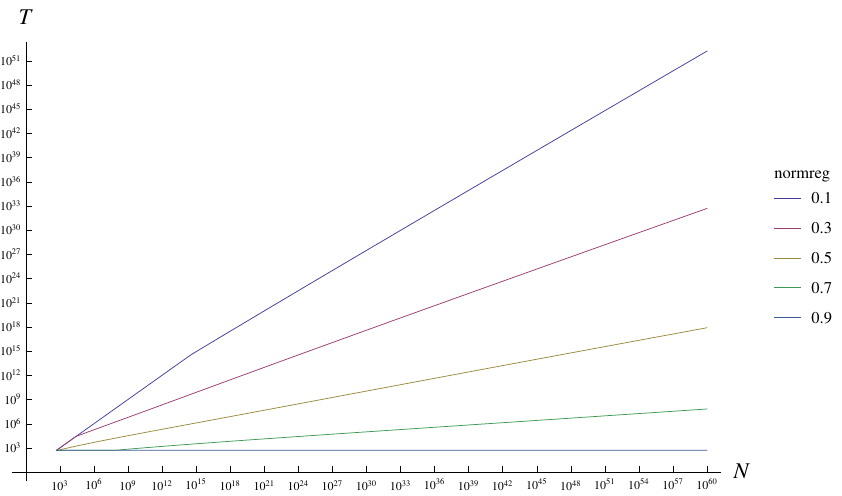}
		\caption{An upper bound on the required number of evaluations $T$ as a function
			of the domain size $N$ (implied by \ref{coro:maximization}). Each
			curve corresponds to a given regret which we want to achieve.}
		\label{fig:required_steps_plot} 
	\end{minipage}
\end{figure}
For instance, if we pick the curve that corresponds to $N=10^{20}$
(the rightmost curve), and we choose $T=10^{6}$, we achieve a regret
of about $0.6$, as indicated in the figure. This means that we can
expect to find a function value of about $40\%$ of the expected global
maximum by evaluating the function at only a fraction of $10^{-14}$
of its domain, and this holds for any prior covariance $\Sigma$.

In \ref{fig:required_steps_plot}, we plot the upper bound on the required
number of evaluations $T$, implied by \ref{coro:maximization}, as
a function of the domain size $N$. 
We observe that $T$ seems to scale polynomially with
$N$ (i.e. linearly in log space) with an order that depends on the
regret we want to achieve. Indeed, it is easy to see from \ref{coro:maximization}
that $\ensuremath{\forall\epsilon>0~\exists K:\forall N\ge T\ge K:}$ 
\begin{equation}
\text{normreg}\le1-\frac{\sqrt{\log T}}{\sqrt{\log N}}+\epsilon \label{eq:asymptotic_regret}
\end{equation}
which implies that $\ensuremath{\forall\epsilon>0~\exists K:\forall N\ge T \ge K:}$
\begin{equation}
T\le N^{(1-\text{normreg}+\epsilon)^{2}}. \label{eq:asymptotic_T}
\end{equation}
This means for instance that, if we accept to achieve only $20\%$
of the global maximum, the required number of evaluations grows very
slowly with about $T\approx N^{1/25}$, but still polynomially. 

\subsubsection{Lower Regret Bound for the Optimal Policy}
The upper bound from \ref{coro:maximization} is tight, as we can see by comparing \ref{eq:asymptotic_regret} to the following result: 
\begin{restatable}[Lower Bound]{lemma}{lowerregretbound}\label{lemma:lowerregretbound} For the instance of the
	problem defined in \ref{sec:po_problem_statement} with $\mu=\mathbf{0}$
	and $\Sigma=\mathbf{I}$, the following lower bound on the regret holds for
	the optimal strategy: $\forall\epsilon>0~\exists K:\forall N\ge T\ge K:$
	\begin{equation}
	\op{normreg}\ge1-\frac{\sqrt{\log T}}{\sqrt{\log N}}-\epsilon.
	\end{equation}
\end{restatable}
\begin{proof}
Here, sampling without replacement is an optimal strategy, which yields the bound above, see \ref{app:lowerregretbound} for the full proof.
\end{proof}
It follows from \Cref{lemma:lowerregretbound} that $\ensuremath{\forall\epsilon>0~\exists K:\forall N\ge T\ge K:}$
\begin{equation}
T\ge N^{(1-\text{normreg}-\epsilon)^{2}}.
\end{equation}
Hence, for a given regret, the required number of evaluations $T$ grows
polynomially  in the domain size $N$, even for the optimal policy, albeit with low degree. This means
that if the domain size grows exponentially in the problem dimension,
we will inherit this exponential growth also for the necessary number
of evaluations.
However, since our bounds are problem independent and hence worst-case in terms of the prior covariance $\Sigma$, this result does
not exclude the possibility that for certain covariances we might
be able to obtain polynomial scaling of the required
number of evaluations in the dimension of the problem.

\subsubsection{Lower Regret Bound for Prior-Independent Policies}

It is important to note that a uniform random policy will not achieve
the bound from \ref{coro:maximization} in general. In fact, despite the bound from \ref{coro:maximization} being
independent of the prior $(\mu, \Sigma)$, it cannot be achieved by any policy that
is independent of the prior:

\begin{restatable}{lemma}{priordependentpolicies}\label{lemma:priordependentpolicies}
	
	For any optimization policy which does not depend on the prior $(\mu,\Sigma)$,
	there exists an instance of the problem defined in \ref{sec:po_problem_statement}
	where 
	\begin{equation}
	\op{normreg}\ge1-\frac{T}{N}\label{blablalladfa},
	\end{equation}
	which is clearly worse for $T\ll N$ than the bound in \ref{coro:maximization}.
\end{restatable}

\begin{proof}
	Suppose we construct a function that is zero everywhere except in one location. Since the policy has no knowledge of that location, it is possible to place it such that the policy will perform no better than random selection, which yields the regret above. See \ref{app:priordependentpolicies} for the full proof.
\end{proof}

\subsection{Extension to Continuous Domains}

For finite domains, we looked at the setting where the cardinality
of the domain $N$ is much larger than the number of admissible evaluations
$T$. The notion of domain size is less obvious in the continuous
case, in the following we clarify this point before we discuss the results.

\subsubsection{Problem Setting}

In the continuous setting, we characterize the GP by $L_{k}$ and
$\sigma$, which are properties of the kernel $k$: 
\begin{align}
	|k(x,x)-k(x,y)|&\le L_{k}\left\Vert x-y\right\Vert _{\infty}\quad\forall x,y\in\mathcal{A}\\
k(x,x) & \le\sigma^{2}\quad\forall x\in\mathcal{A}
\end{align}%
where $\mathcal{A}$ is the $D$-dimensional unit cube (note that
to use a domain other than the unit cube we can simply rescale). The
setting we are interested in is 
\begin{equation}
T\ll\left(\frac{L_{k}}{2\sigma^{2}}\right)^{D}=:m(L_{k},\sigma,D).
\end{equation}
As we discuss in more detail in \ref{sec:cont_setting}, $m(L_{k},\sigma,D)$ corresponds to the number of points 
we would require to cover the domain such that we could acquire nonzero 
information about any point in the domain. Naturally, to guarantee that 
we can find the global optimum we would require at least that number of evaluations $T$.
Here, in contrast, we consider the setting where $T$ is much smaller and only a small fraction
of the GP can be explored.

\subsubsection{Results}
Here, we adapt \ref{coro:maximization} to continuous domains. The
bound we propose in the following is based on a result from \citep{Grunewalder2010-uu}, which states that for a centered Gaussian Process $\left(G_{a}\right){}_{a\in\mathcal{A}}$
with domain $\mathcal{A}$ being the $D$-dimensional unit cube and a
Lipschitz-continuous kernel $k$ 
\begin{equation}
\left|k(x,x)-k(x,y)\right|\le L_{k}\left\Vert x-y\right\Vert _{\infty}\quad\forall x,y\in\mathcal{A},
\end{equation}
the regret of the optimal policy is bounded by 
\begin{align}
\mathbb{E}\left[\sup_{a\in\mathcal{A}}G(a)-\hat{Y}\right] & \le\sqrt{\frac{2L_{k}}{\left\lfloor T^{1/D}\right\rfloor }}\left(2\sqrt{\log\left(2T\right)}+15\sqrt{D}\right).\label{eq:grunewalder}
\end{align}
Note that \citet{Grunewalder2010-uu} state the bound for the more
general case of Hölder-continuous functions, for simplicity of exposition
we limit ourselves here to the case of Lipschitz-continuous kernels.
\citet{Grunewalder2010-uu} complement this bound with a matching
lower bound (up to log factors). However, as we shall see, we can improve
substantially on this bound in terms of its dependence on the Lipschitz
constant $L_{k}$ if we assume that the variance is bounded, i.e., $k(x,x)\le\sigma^{2}$.
This is particularly relevant for GPs with short length scales (or,
equivalently, large domains) and hence large $L_{k}$.

Interestingly, \citet{Grunewalder2010-uu} obtain \ref{eq:grunewalder} using a policy that selects
the actions a priori (by placing them on a grid), without any feedback
from the observations made. Here, we will refine \ref{eq:grunewalder}
by using this strategy for preselecting
a large set of admissible actions offline and then selecting actions from this
set using EI2 (\ref{def:ei_extremization}) or UCB2 (\ref{def:ucb_extremization})
online. A reasoning along these lines yields the following bound:

\begin{restatable}{theorem}{continuoustheorem}\label{thm:continuous}

For any centered Gaussian Process $\left(G_{a}\right){}_{a\in\mathcal{A}}$,
where $\mathcal{A}$ is the $D$-dimensional unit cube, with kernel
$k$ such that 
\begin{align}
\left|k(x,x)-k(x,y)\right| & \le L_{k}\left\Vert x-y\right\Vert _{\infty}\quad\forall x,y\in\mathcal{A}\\
\sqrt{k(x,x)} & \le\sigma\quad\forall x\in\mathcal{A}
\end{align}
we obtain the following bound on the regret, if we follow the
EI2 (\ref{def:ei_extremization}) or the UCB2 (\ref{def:ucb_extremization})
strategy: 
\begin{align}
  &\mathbb{E}\left[\sup_{a\in\mathcal{A}}G(a)-\hat{Y}\right]\le\nonumber \sqrt{\frac{2\log(L_{k})}{T^{1/D}}}\left(2\sqrt{\log\left(2\left\lceil \frac{L_{k}}{\log(L_{k})}T^{1/D}\right\rceil ^{D}\right)}+15\sqrt{D}\right)+\nonumber \\
 & \quad\quad\quad\quad\quad\sqrt{2}\sigma\left(\sqrt{D\log\left(\left\lceil \frac{L_{k}}{\log(L_{k})}T^{1/D}\right\rceil \right)}-\left(1-T^{-\frac{1}{2\sqrt{\pi}}}\right)\sqrt{\log\left(\frac{T}{3\log^{\frac{3}{2}}(T)}\right)}\right).\label{eq:continuous_bound}
\end{align}
This bound also holds when restricting EI2 or UCB2 to a uniform grid
on the domain $\mathcal{A}$, where each side is divided into $\left\lceil \frac{L_{k}}{\log(L_{k})}T^{1/D}\right\rceil $
segments. Finally, this bound converges to $0$ as $T\to\infty$.

\end{restatable} 
\begin{proof}
The idea here is to pre-select a set of $N$ points at locations $X_{1:N}$
on a grid and then sub-select points from this set during runtime using EI2 (\ref{def:ei_extremization}) or UCB2 (\ref{def:ucb_extremization}).
We bound the regret of this strategy by combining \ref{theorem:main} with the main result from \citep{Grunewalder2010-uu},
the full proof can be found in \ref{app:continuous}.
\end{proof}
The important point to note here is that in \ref{thm:continuous}
the bound grows logarithmically in $L_{k}$, as opposed to the bound
\ref{eq:grunewalder} from \cite{Grunewalder2010-uu}, which grows
with $\sqrt{L_{k}}$. This means that for Gaussian Processes with
high $L_{k}$, i.e. high variability, \ref{eq:continuous_bound} is much lower than \ref{eq:grunewalder} (note that cuboid domains $\mathcal{A}$ can be rescaled to the unit cube by adapting the Lipschitz constant $L_k$ accordingly, hence a large domain is equivalent to a large Lipschitz constant). This allows for meaningful bounds even when the number of allowed evaluations is small relative to the domain-size of the function.
We illustrate this in \ref{fig:continuous}.
Consider for instance the case of $L_{k}=10^{5},\sigma=1,D=2,T=10^{5}$,
where the number of evaluations is far too low to explore the GP ($T=10^{5}\ll m(L_{k},\sigma,D)=2.5\times10^{9}$).
We see from \ref{fig:continuous_D} that our regret bounds (second-darkest
cyan) remain low while the ones from \cite{Grunewalder2010-uu} (second-darkest
magenta) explode.
\begin{figure}
	\centering
	\begin{subfigure}{.45\textwidth}
	  \centering
	  \includegraphics[width=\linewidth]{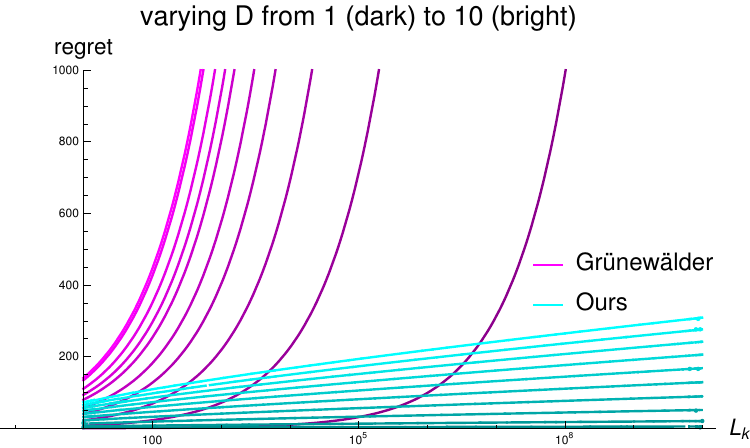}
	  \caption{Different shades correspond to different dimensions $D\in\{1,...,10\}$ from dark to bright, the other parameters are:\\ $\alpha =1, T=10^5, \sigma=1$.}
	  \label{fig:continuous_D}
	\end{subfigure}%
	\hspace{0.5cm}
	\begin{subfigure}{.45\textwidth}
	  \centering
	  \includegraphics[width=\linewidth]{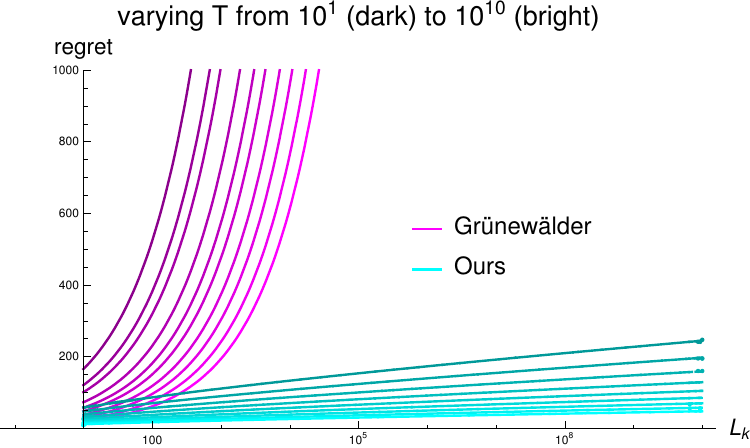}
	  \caption{Different shades correspond to different numbers of evaluations $T\in\{10^1,10^2,...,10^{10}\}$ from dark to bright, the other parameters are: \\$\alpha =1, D=5, \sigma=1$.}
	  \label{fig:continuous_T}
	\end{subfigure}
	\caption{The regret as a function of the Lipschitz constant $L_k$. Comparison of the bound from \cite{Grunewalder2010-uu} (magenta) and ours \ref{eq:continuous_bound} (cyan).}
	\label{fig:continuous}
\end{figure}
\ref{thm:continuous} also provides another insight: To allow for straightforward optimization of the acquisition function (e.g. EI, UCB), the domain is often discretized in practical Bayesian optimization. \ref{thm:continuous} tells us how fine this discretization should be to still achieve performance guarantees.

% Preview source code from paragraph 0 to 16

\section{Relation to Standard EI and UCB} \label{sec:relation_to_standard}

An empirical comparison (\ref{sec:experiments}) indicates EI2/UCB2
require more evaluations than EI/UCB to attain a given regret, but
not more than twice as many. This matches our intuition: We would
expect standard EI/UCB to perform better because i) any given step,
evaluating at a potential maximizer instead of a minimizer will clearly
lead to a larger immediate reduction in expected regret and ii) we
would not expect an evaluation at a potential minimizer to provide
any more useful information than an evaluation at a potential maximizer.
Further, we would not expect EI2/UCB2 to perform much worse because
a substantial fraction of its evaluations (in expectation half, for
a centered GP) will be maximizations.

If we could prove that EI/UCB performs better than EI2/UCB2 for maximization,
this would imply that the regret bounds presented above apply to EI/UCB. Unfortunately, proving
this formally appears to be nontrivial. Nevertheless, we are able
to give weaker regret bounds for standard EI/UCB which we discuss
in the following. 

\subsection{Upper Bound on the Bayesian Simple Regret for Standard EI and UCB}

Let us now consider the standard, one-sided versions of EI (\ref{def:ei_extremization_standard}) and UCB (\ref{def:ucb_extremization_standard}), which are identical to EI2 (\ref{def:ei_extremization}) and UCB2 (\ref{def:ucb_extremization}) except that we drop the second term in the $\max$. We obtain the following version of \ref{theorem:main}:

\begin{restatable}{theorem}{maintheoremstandard}\label{theorem:main_standard}For any instance
	$(N,T,\mu,\Sigma)$ of the problem defined in \ref{sec:po_problem_statement} with $N\ge T\ge500$, if
we follow either the $\text{\textbf{EI}}$ (\ref{def:ei_extremization_standard}) or the $\text{\textbf{UCB strategy}}$ (\ref{def:ucb_extremization_standard})
we have 
\begin{equation}
\frac{\mathbb{E}\left[\hat{\check{F}}\right]-\mathbb{E}\left[\boldsymbol{\hat{Y}-\check{F}}\right]}{\mathbb{E}\left[\hat{\check{F}}\right]}\le1-\left(1-T^{-\frac{1}{2\sqrt{\pi}}}\right)\sqrt{\frac{\log(T)-\log\left(3\log^{\frac{3}{2}}(T)\right)}{\log(N)}}.
\end{equation}

\end{restatable} 

\begin{proof} The proof is very similar to the one of \ref{theorem:main}
and can be found in \ref{sec:proof_standard}. \end{proof}

Note that we marked the changes in bold. Now, instead of a guarantee
on $\hat{\check{Y}}$, we provide a guarantee on $\hat{Y}-\check{F}$,
i.e. the difference between the best obtained value and the function
minimum. We can then derive from that a version of \ref{coro:maximization}:

\begin{corollary}\label{coro:maximization_standard}

	For any instance $(N,T,\mu,\Sigma)$
	of the problem defined in \ref{sec:po_problem_statement} with zero
	mean $\mu_{n}=0~\forall n$ and $N\ge T\ge500$, if we follow either $\text{\textbf{standard EI}}$
(\ref{def:ei_extremization_standard}) or $\text{\textbf{UCB strategy}}$
(\ref{def:ucb_extremization_standard}), we have 
\begin{equation}
normreg\le\boldsymbol{2}\left(1-\left(1-T^{-\frac{1}{2\sqrt{\pi}}}\right)\sqrt{\frac{\log(T)-\log\left(3\log^{\frac{3}{2}}(T)\right)}{\log(N)}}\right).
\end{equation}

\end{corollary}

\begin{proof} The proof is analogous to the one of \ref{coro:maximization}.

\end{proof}

The important thing to note here is the appearance of the factor 2
in the bound. This means that asymptotically we have 
\begin{equation}
\text{normreg}\le\boldsymbol{2}\left(1-\frac{\sqrt{\log T}}{\sqrt{\log N}}\right)+\epsilon
\end{equation}
and 
\begin{equation}
T\le N^{(1-\text{normreg}/\boldsymbol{2}+\epsilon)}
\end{equation}
which is weaker compared to \ref{eq:asymptotic_regret} and \ref{eq:asymptotic_T}.
We believe that this gap is not due to EI/UCB actually performing
worse than EI2/UCB2, but rather an artifact of the proof.

\section{Limitations}\label{sec:limitations}
While we believe that the results above are insightful, there are a number of limitations one should be aware of:

As discussed in \ref{sec:relation_to_standard}, the regret bounds we derive for standard EI and UCB are weaker than the ones for EI2 and UCB2, despite the intuition and empirical evidence that EI/UCB most likely perform no worse for maximization than EI2/UCB2. It would be interesting to close this gap.

Another limitation is that the bounds only hold for the noise-free setting. We believe that this limitation is acceptable because the problem of noisy observations is mostly orthogonal to the problem studied herein. Furthermore, the naive solution of reducing the noise by evaluating multiple times at each point leads to qualitatively similar regret bounds, see \ref{sec:observation_noise} for a more detailed discussion.

Further, the bounds from \ref{coro:maximization}, \ref{thm:continuous}, and \ref{coro:maximization_standard} only hold for zero prior mean (equivalently, constant prior mean). This assumption is not uncommon in the GP optimization literature (see e.g. \cite{gpucb}) but it may be limiting if one has prior knowledge about where good function values lie. It is likely possible to extend the results in this article to arbitrary prior means.

Finally, there are limitations that hold generally for the GP optimization literature: 1) In a naive implementation, the computational cost is cubic in the number of evaluations $T$ and 2) the assumption that the true function is drawn from a Gaussian Process is typically not realistic and only made for analytical convenience. It is hence not clear whether the relations we uncovered herein apply to realistic optimization settings or if they are mostly an artifact of the GP assumption.

Summarizing, it is clear that the results in the present paper have little direct practical relevance. Instead, the intention is to develop a theoretical understanding of the problem setting.

\section{Conclusion}

We have characterized GP optimization in the setting where finding the global optimum is impossible because the number of evaluations is too small with respect to the domain size. We derived regret-bounds for the finite-arm setting which are independent of the prior covariance, and we showed that they are tight. Further, we derived regret-bounds for GP optimization in continuous domains that depend on the Lipschitz constant of the covariance function and the maximum variance.
In contrast to previous work, our bounds are non-vacuous even when the domain size is very large relative to the number of evaluations.
Therefore, they provide novel insights into the performance of GP optimization in this challenging setting.
In particular, they show that even when the number of evaluations is far too small to find the global optimum, we can find nontrivial function values (e.g. values that achieve a certain ratio with the optimal value).

\newpage

% Acknowledgments---Will not appear in anonymized version
\bibliographystyle{plainnat}

\bibliography{references}

\newpage

\appendix

\section{Relation to Adaptive Submodularity}\label{app:as}

% Anteprima del sorgente dal paragrafo 61 al 79
Submodularity is a property of set functions with far-reaching implications.
Most importantly here, it allows for efficient approximate optimization
\cite{nemhauser1978}, given the additional condition of monotonicity.
This fact has been exploited in many information gathering applications,
see \cite{krause14survey} for an overview.

In \cite{golovin11adaptive}, the authors extend the notion of submodularity
to adaptive problems, where decisions are based on information acquired
online. This is precisely the setting we consider in the present paper.
However, as we will show shortly, our function maximization problem
is not submodular. Nevertheless, our proof is inspired by the notion
of adaptive sumodularity.

Consider the following definitions, which we adapted to the notation
in the present paper: 

\begin{definition}(\citet{golovin11adaptive}) The Conditional Expected
	Marginal Benefit with respect to some utility function $u$ is defined
	as 
	\begin{equation}
	\Delta_{u}(a|a_{1:t},y_{1:t})=\mathbb{E}\left[u(F,A_{1:t+1})-u(F,A_{1:t})|Y_{1:t}=y_{1:t},A_{1:t}=a_{1:t},A_{t+1}=a\right].
	\end{equation}
\end{definition}

\begin{definition}\label{def:adaptive_submodular}(\citet{golovin11adaptive})
	Adaptive submodularity holds if for any $t\le k\in\mathbb{N}$, any
	$a_{1:k},y_{1:k}$ and any $a$ we have
	\begin{equation}
	\Delta_{u}(a|a_{1:t},y_{1:t})\ge\Delta_{u}(a|a_{1:k},y_{1:k}).
	\end{equation}
\end{definition} Intuitively, in an adaptively submodular problem the expected
benefit of any given action $a$ decreases the more information we
gather. \citet{golovin11adaptive} show that if a problem is adaptively
submodular (along with some other condition), then the greedy policy
will converge exponentially to the optimal policy. 

\subsection{Gaussian-Process Optimization is not Adaptively Submodular}

In the following we make a simple argument why GP optimization is
not generally adaptively submodular. It is not entirely clear what
is the right utility function $u$, but our argument holds for any
plausible choice. 

Consider a function $F_{1:N}$ with all values mutually independent,
except for $F_{1}$ and $F_{2}$ which are negatively correlated.
Further, suppose that we made an observation $y_{1}$ which is far
larger than the upper confidence bounds on $F_{1}$ and $F_{2}$.
Any reasonable choice of utility function would yield an extremely
small conditional expected marginal benefit for $A_{2}=2$, since
we would not expect this to give us any information about the optimum.
Now suppose we evaluate the function at $A_{2}=1$ and observe a $y_{2}$
such that the posterior mean of $F_{2}$ is approximately equal to
$y_{1}$. Now, the conditional expected marginal benefit of evaluating
at $A_{3}=2$ should be substantial for any reasonable utility, since
the maximum might lie at that point. 
More generally, through unlikely observations the GP landscape can change completely and points which seemed uninteresting before can become interesting, which violates the diminishing-returns property of adaptive submodularity \ref{def:adaptive_submodular}.

\section{Relation to GP-UCB}\label{sec:gpucb}

The bounds from \cite{gpucb} are only
meaningful if we are allowed to make a sufficient number of evaluations to attain low uncertainty
over the entire GP. 
% Preview source code from paragraph 1 to 19
The reason is that these bounds
depend on a term called the information gain $\gamma_T$,
which represents the maximum information that can be acquired about the GP with $T$ 
evaluations. As long as the GP still has large uncertainty in some areas,
each additional evaluation may add a substantial amount of information (there is no saturation) and $\gamma_T$, and hence the cumulative regret, will keep growing.

To see this, consider Lemma 5.3 in \cite{gpucb} (we use a slightly different notation here):
The information gain of a set of points $X={x_{1},...,x_{T}}$
can be expressed as
\begin{equation}
G(X):=I(y_{X};f_{X})=\frac{1}{2}\sum_{t=1}^{T}\log(1+\sigma_{y}^{-2}\sigma^{2}(x_{t}|x_{1:t-1}))
\end{equation}
where $\sigma^{2}(x_{t}|x_{1:t-1})$ is the predictive variance after
evaluating at $x_{1:t-1}$ and $\sigma_{y}^{2}$ is the variance of
the observation noise\footnote{Note that the information gain goes to infinity as the observation
noise $\sigma_{y}$ goes to zero, which is in fact another reason
why the results from \cite{gpucb} are not directly applicable to
our setting. However, this is a technicality that can be resolved (in the most naive way, one could add artificial noise).}. Hence, we can write the information
gain for $T+1$ points as
\begin{equation}
G(X\cup\{x_{T+1}\})=G(X)+\frac{1}{2}\log(1+\sigma_{y}^{-2}\sigma^{2}(x_{T+1}|x_{1:T})). \label{eq:info_gain_sum}
\end{equation}
Now let $X^{*}:=\max_{X:|X|=T}G(X)$ be the points that maximize the
information gain. By definition (see equation 7 in \cite{gpucb}),
we have
\begin{align}
\gamma_{T} & :=G(X^{*})
\end{align}
that is, $\gamma_{T}$ is the maximum information that can be acquired
using $T$ points. For $T+1$ points we have
\begin{align}
\gamma_{T+1} & =\max_{X,x_{T+1}}G(X\cup\{x_{T+1}\})\\
 & \ge\max_{x_{T+1}}G(X^{*}\cup\{x_{T+1}\})
\end{align}
where the inequality follows from the fact that maximizing over $X,x_{T+1}$
jointly will at least yield as high a value as just picking $X^{*}$
from the previous optimization and optimizing only over $x_{T+1}$.
Plugging in \ref{eq:info_gain_sum}, we have
\begin{equation}
\gamma_{T+1}\ge G(X^{*})+\frac{1}{2}\max_{x_{T+1}}\log(1+\sigma_{y}^{-2}\sigma^{2}(x_{T+1}|x_{1:T}^{*}))
\end{equation}
and hence 
\begin{equation}
\gamma_{T+1}\ge\gamma_{T}+\frac{1}{2}\max_{x_{T+1}}\log(1+\sigma_{y}^{-2}\sigma^{2}(x_{T+1}|x_{1:T}^{*})).
\end{equation}
This means that if $T$ is not large enough to explore the GP reasonably
well everywhere (i.e., there are still $x$ such that $\sigma^{2}(x|x_{1:T}^{*})$
is large), then adding an observation can add substantial
information, i.e. $\gamma_{T+1}$ is substantially larger than $\gamma_{T}$
(which means the regret grows substantially).

As a more concrete case, suppose we have a GP which a priori has a
uniform variance $\sigma^{2}(x)=s^{2}\quad\forall x$. In addition,
suppose that the GP domain is large with respect to $T$, in the sense
that it is not possible to reduce the variance everywhere substantially
by observing $T$ (or less) points, i.e. we have $\max_{x_{t}}\sigma(x_{t}|x_{1:t-1})\approx s~\forall x_{1:t-1},t\le T+1$.
We hence have
\begin{align}
\gamma_{T}&=\max_{x_{1:T}}\frac{1}{2}\sum_{t=1}^{T}\log(1+\sigma_{y}^{-2}\sigma^{2}(x_{t}|x_{1:t-1}))\\
&\approx\frac{1}{2}T\log(1+\sigma_{y}^{-2}s^{2}).
\end{align}

This linear growth will continue until $T$ is large enough such that
the uncertainty of the GP can be reduced substantially everywhere.

Since the bound on the cumulative regret $R_{T}$ is of the form $\sqrt{T\gamma_{T}}$
(see Theorem 1 in \cite{gpucb}) it will hence also grow linearly
in $T$. \cite{gpucb} then bound the suboptimality of the optimization
by the average regret $R_{T}/T$ (see the paragraph on regret in Section
2 of \cite{gpucb}), which does not decrease as long as $R_{T}$
grows linearly in $T$.

\section{The Continuous-Domain Setting}\label{sec:cont_setting}
In the continuous setting, we characterize the GP by $L_{k}$ and
$\sigma$, which are properties of the kernel $k$: 
\begin{align}
	|k(x,x)-k(x,y)|&\le L_{k}\left\Vert x-y\right\Vert _{\infty}\quad\forall x,y\in\mathcal{A}\\
k(x,x) & \le\sigma^{2}\quad\forall x\in\mathcal{A}
\end{align}
where $\mathcal{A}$ is the $D$-dimensional unit cube (note that
to use a domain other than the unit cube we can simply rescale). The
setting we are interested in is 
\begin{equation}
T\ll\left(\frac{L_{k}}{2\sigma^{2}}\right)^{D}=:m(L_{k},\sigma,D),
\end{equation}
which implies that a large part of the GP may remain unexplored, as
will become clear in the following comparison to related work:

As discussed in the introduction and in \ref{sec:gpucb}, the results from \cite{gpucb} only apply when we can reduce the maximum variance
of the GP using $T$ evaluations. This would
require that we can acquire information on each point $x$ that has
maximum prior variance $k(x,x)=\text{max}_{z}k(z,z)=\sigma^{2}$.
In order to ensure that we gather nonzero information on such a point
$x$, we have to make sure to evaluate at least one point $y$ such
that $k(x,y)>0$ (or, more realistically, $k(x,y)>\epsilon$, which
would lead to a qualitatively similar result), which is equivalent
to the condition 
\begin{equation}
|k(x,x)-k(x,y)|<k(x,x)=\sigma^{2},
\end{equation}
which we can ensure by 
\begin{equation}
L_{k}\left\Vert x-y\right\Vert _{\infty}<\sigma^{2}
\end{equation}
or equivalently by 
\begin{equation}
\left\Vert x-y\right\Vert _{\infty}<\frac{\sigma^{2}}{L_{k}}.
\end{equation}
This statement says that $x$ has to be within a cube centered at
$y$ with sidelength $2\sigma^{2}/L_{k}$. To ensure that this holds
for all $x\in\mathcal{A}$ (since in the worst case they all have
prior variance $\sigma^{2}$, which is typical), we need to cover
the domain with 
\begin{equation}
T>\left(\frac{L_{k}}{2\sigma^{2}}\right)^{D}=m(L_{k},\sigma,D)
\end{equation}
cubes and hence evaluations.

\section{A Note on Observation Noise}\label{sec:observation_noise}

Our goal here was to focus on the issue of large domains, without
the added difficulty of noisy observations, such as to allow a clearer
view of the core problem. Interestingly, the proofs apply practically
without any changes to the setting with observation noise. The caveat
is that the regret bounds are on the largest $\text{\textbf{noisy observation}}$
$\hat{Y}$ rather than the largest retrieved $\text{\textbf{function value}}$
$\max_{t}F_{A_{t}}$ (the two are identical in the noise-free setting).

As a naive way of obtaining regret bounds on $\max_{t}F_{A_{t}}$,
one could simply evaluate each point $n$ times and use the average
observation as a pseudo observation. Choosing $n$ large enough, all
pseudo observations $Y_{1:T}$ will be close to their respective function
values $F_{A_{1:T}}$ with high probability. To guarantee that all
$T$ pseudo observations are within $\epsilon$ of the true function
values with probability $\delta$, we would need $n=\log(T/\delta)f(\sigma_{y},\epsilon)$
(this follows from union bound over $T$ observations), where $f$
is some function that is not relevant here and $\sigma_{y}$ is the
noise standard deviation. We can now simply replace $T$ with $T/\left(\log(T/\delta)f(\sigma_{y},\epsilon)\right)$
in all the theorems (to be precise, we would also have to add $\epsilon$
to the regret, but it can be made arbitrarily small). While this solution
is impractical, it is interesting to note that the dependence of the
resulting regret-bounds on the domain size $N$ and Lipschitz constant
$L_{k}$ does not change. The dependence on $T$ is also identical,
up to a $\log$ factor. This suggests that the relations we uncovered
in this paper between the regret, the number of evaluations $T$,
the domain size $N$, the Lipschitz constant $L_{k}$ remain qualitatively
the same in the presence of observation noise.

\section{Definitions of Standard EI/UCB}

\begin{definition}[EI]\label{def:ei_extremization_standard} An
	agent follows the EI strategy when it picks its actions according
	to 
	\begin{align}
	A_{t+1} & =\op{argmax}_{n\in[N]}\op{ei}\left(\hat{Y}_{t}|M_{t}^{n},\sqrt{C_{t}^{nn}}\right)
	\end{align}
	with the expected improvement $\op{ei}$ as defined in \ref{def:ei}.
	\end{definition}
	
	\begin{definition}[UCB]\label{def:ucb_extremization_standard}
	An agent follows the UCB strategy when it picks its actions according
	to 
	\begin{align}
	A_{t+1} & =\op{argmax}_{n\in[N]}\left(M_{t}^{n}+\sqrt{C_{t}^{nn}2\log N}\right).
	\end{align}
	\end{definition}

\section{Empirical Comparison between EI/UCB and EI2/UCB2}
\label{sec:experiments}

We conducted a number of within-model experiments, where the ground-truth function is a sample drawn from the GP. The continuous-domain experiments use the \cite{gpy2014} library and the discrete-domain experiments use scikit-learn (\cite{Pedregosa2011-hb}). The code for all the experiments is publicly accessible\footnote{\tiny\url{https://github.com/mwuethri/regret-Bounds-for-Gaussian-Process-Optimization-in-Large-Domains}}.

\subsection{Continuous Domain}
We defined a GP $G$ on a $D$-dimensional unit cube with a squared-exponential
kernel with length scale $l$. The smaller the length-scale and the
larger the dimensionality, the harder the problem. For the experiments
that follow, we chose the ranges of $D,l$ such that we cover the
classical setting (where the global optimum can be identified) as
well as the large-domain setting (where this is not possible). To make this scenario computationally tractable, we discretize the domain using a grid with $1000$ points, and we allow the algorithm to evaluate at $T=50$ points.

The true expected regret
\begin{equation}
\mathbb{E}\left[\sup_{a\in\mathcal{A}}G(a)-\hat{Y}\right]=r(l,D,T)
\end{equation}
is a function of the length-scale $l$, the dimension $D$, and the
number of function evaluations $T$. We compute this quantity empirically
using $1000$ samples (i.e. $1000$ randomly-drawn ground-truth functions). In \ref{fig:continuous_ei_vs_ei2_plot}, for instance, we plot this empirical expected regret as a function of the number of evaluations $T$.
\begin{figure}
	\centering
	\includegraphics[width=0.5\linewidth]{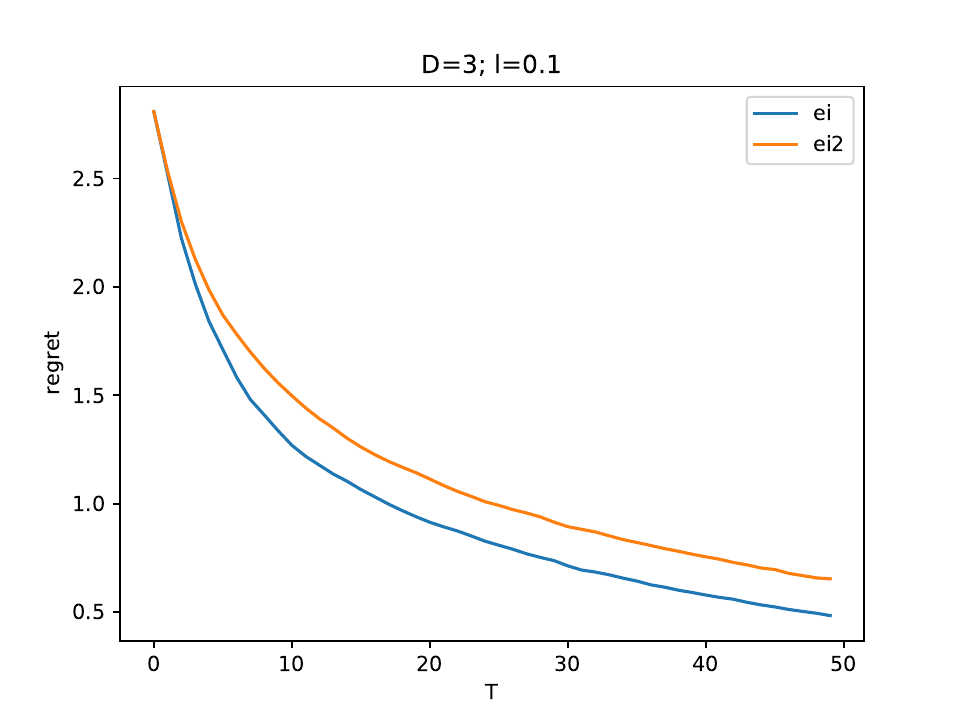}
	\caption{The empirical expected regret as a function of the number of evaluations $T$.}
	\label{fig:continuous_ei_vs_ei2_plot}
\end{figure}
In this example, EI performs slightly better than EI2. To gain a more quantitative understanding, it is instructive to look at how many evaluations $T$
are required to attain a given regret $R=r(l,D,T)$: 
\begin{equation}
T=t(R,l,D).
\end{equation}
We can then compare the required number of steps for EI and EI2:
\begin{equation}
\frac{t_{ei}(R,l,D)}{t_{ei2}(R,l,D)},
\end{equation}
which we report in \ref{fig:continuous_ei_vs_ei2}.
\begin{table}
\centering
\caption{Fraction $\frac{t_{ei}(R,l,D)}{t_{ei2}(R,l,D)}$ for continuous 
    domains. NaN entries correspond to the case where the given regret was 
    not attained after $T=50$ evaluations.}
\label{fig:continuous_ei_vs_ei2}
\begin{tabular}{llrrrrr}
\toprule
  &       &   2.5 &   2.0 &   1.5 &   1.0 &   0.5 \\
D & l &       &       &       &       &       \\
\midrule
1 & 0.003 &  1.00 &  1.00 &  0.87 &  0.84 &   NaN \\
  & 0.010 &  1.00 &  1.00 &  0.80 &  0.80 &  0.74 \\
  & 0.030 &  1.00 &  1.00 &  1.00 &  1.00 &  0.78 \\
  & 0.100 &  1.00 &  1.00 &  1.00 &  1.00 &  0.75 \\
  & 0.300 &  1.00 &  1.00 &  1.00 &  1.00 &  1.00 \\
2 & 0.003 &  1.00 &  1.00 &  1.00 &  1.00 &   NaN \\
  & 0.010 &  1.00 &  0.86 &  1.00 &   NaN &   NaN \\
  & 0.030 &  0.67 &  1.00 &  0.91 &  0.71 &   NaN \\
  & 0.100 &  1.00 &  1.00 &  0.80 &  0.78 &  0.75 \\
  & 0.300 &  1.00 &  1.00 &  1.00 &  1.00 &  0.83 \\
3 & 0.003 &  1.00 &  1.00 &  1.00 &  1.00 &   NaN \\
  & 0.010 &  1.00 &  1.00 &  1.00 &   NaN &   NaN \\
  & 0.030 &  1.00 &  1.00 &  1.06 &   NaN &   NaN \\
  & 0.100 &  1.00 &  1.00 &  0.73 &  0.69 &   NaN \\
  & 0.300 &  1.00 &  1.00 &  0.75 &  0.86 &  0.73 \\
4 & 0.003 &  1.00 &  1.00 &  1.00 &   NaN &   NaN \\
  & 0.010 &  1.00 &  1.00 &  1.00 &   NaN &   NaN \\
  & 0.030 &  1.00 &  1.00 &  1.00 &   NaN &   NaN \\
  & 0.100 &  1.00 &  1.00 &  0.84 &   NaN &   NaN \\
  & 0.300 &  1.00 &  0.75 &  0.86 &  0.71 &  0.70 \\
\bottomrule
\end{tabular}
\end{table}

In one entry EI2 appears to perform slightly better, we have $t_{ei2} = 1.06 t_{ei}$. This is may be due to the variance of the empirical estimation.
In all other entries, we have $0.5t_{ei2}\le t_{ei}\le t_{ei2}$,
which means that EI always reaches the given expected regret $R$
faster than EI2, but not more than twice as fast. 
This is what we
intuitively expected: EI should do better than EI2, because it does
not waste evaluations on minimization, but not much better, since in expectation every second
evaluation of EI2 is a maximization. Note that the entries which are
$1$, i.e. both algorithms perform equally well, correspond to i) particularly
simple settings (large $l$, low $D$) where both algorithms find
good values in just a handful of evaluations or ii) particularly hard settings where there is no essentially no correlation between different points in the discretized domain.

For UCB and UCB2 (where we used a fixed confidence level) we obtain
similar results, see \ref{fig:continuous_ucb_vs_ucb2}.
\begin{table}
\centering
\caption{Fraction $\frac{t_{ucb}(R,l,D)}{t_{ucb2}(R,l,D)}$ for 
    continuous domains. NaN entries correspond to the case where the 
    given regret was not attained after $T=50$ evaluations.}
\label{fig:continuous_ucb_vs_ucb2}
\begin{tabular}{llrrrrr}
\toprule
  &       &   2.5 &   2.0 &   1.5 &   1.0 &   0.5 \\
D & l &       &       &       &       &       \\
\midrule
1 & 0.003 &  1.00 &  1.00 &  0.89 &  0.86 &   NaN \\
  & 0.010 &  1.00 &  1.00 &  0.80 &  0.90 &  0.81 \\
  & 0.030 &  1.00 &  1.00 &  1.00 &  1.00 &  0.90 \\
  & 0.100 &  1.00 &  1.00 &  1.00 &  1.00 &  1.00 \\
  & 0.300 &  1.00 &  1.00 &  1.00 &  1.00 &  1.00 \\
2 & 0.003 &  1.00 &  1.00 &  1.00 &  1.00 &   NaN \\
  & 0.010 &  1.00 &  0.86 &  1.00 &   NaN &   NaN \\
  & 0.030 &  0.67 &  0.83 &  0.83 &  0.77 &   NaN \\
  & 0.100 &  1.00 &  1.00 &  0.80 &  0.73 &  0.78 \\
  & 0.300 &  1.00 &  1.00 &  1.00 &  1.00 &  0.83 \\
3 & 0.003 &  1.00 &  1.00 &  1.00 &  1.00 &   NaN \\
  & 0.010 &  1.00 &  1.00 &  1.00 &   NaN &   NaN \\
  & 0.030 &  1.00 &  1.00 &  1.00 &   NaN &   NaN \\
  & 0.100 &  1.00 &  0.83 &  0.75 &  0.76 &   NaN \\
  & 0.300 &  1.00 &  1.00 &  0.75 &  1.00 &  0.71 \\
4 & 0.003 &  1.00 &  1.00 &  1.00 &   NaN &   NaN \\
  & 0.010 &  1.00 &  1.00 &  1.00 &   NaN &   NaN \\
  & 0.030 &  1.00 &  1.00 &  1.00 &   NaN &   NaN \\
  & 0.100 &  1.00 &  0.86 &  0.83 &   NaN &   NaN \\
  & 0.300 &  1.00 &  0.75 &  0.87 &  0.75 &  0.66 \\
\bottomrule
\end{tabular}
\end{table}

\subsection{Band Covariance Matrices}
Next, we compare EI/UCB with EI2/UCB2 in the discrete setting with $N=100$. We use band covariance matrices, where the diagonal elements are equal to $1$ and there are a number of nonzero elements to the right and the left of the diagonal. We vary width of this band and the value the off-diagonal elements take, we report the results in \ref{fig:band_ei_vs_ei2} for EI vs EI2 and in \ref{fig:band_ucb_vs_ucb2} for UCB vs UCB2. Similarly to the case of continuous domains, we see that $0.5t_{ei2}\le t_{ei}\le t_{ei2}$ (and the equivalent for UCB).
\begin{table}
\centering
\caption{Fraction $\frac{t_{ei}(\text{band\_size, 
    band\_corr})}{t_{ei2}(\text{band\_size, band\_corr})}$ 
    for finite band covariance matrices. The covariance matries are identity 
    matrices with $\text{band\_size}$ many elements with value $\text{band\_corr}$ 
    added to each side of the diagonal. NaN entries correspond to the 
    case where the given regret was not attained after $T=20$ evaluations.}
\label{fig:band_ei_vs_ei2}
\begin{tabular}{llrrrrr}
\toprule
   &       &  2.00 &  1.55 &  1.10 &  0.65 &  0.20 \\
band\_size & band\_corr &       &       &       &       &       \\
\midrule
0  &  0.00 &   1.0 &  1.00 &  1.00 &   NaN &   NaN \\
2  & -0.20 &   1.0 &  0.75 &  0.86 &  0.82 &   NaN \\
3  &  0.20 &   1.0 &  1.00 &  0.87 &  0.89 &   NaN \\
5  & -0.10 &   1.0 &  1.00 &  1.00 &  0.94 &   NaN \\
   &  0.20 &   1.0 &  1.00 &  1.00 &  0.88 &   NaN \\
10 &  0.10 &   1.0 &  1.00 &  1.00 &  0.94 &   NaN \\
40 &  0.05 &   1.0 &  1.00 &  0.87 &  0.94 &   NaN \\
\bottomrule
\end{tabular}
\end{table}

\begin{table}
\centering
\caption{Fraction $\frac{t_{ucb}(\text{band\_size, band\_corr})}
    {t_{ucb2}(\text{band\_size, band\_corr})}$ for finite band covariance 
    matrices. The covariance matries are identity matrices with 
    $\text{band\_size}$ many elements with value $\text{band\_corr}$ added 
    to each side of the diagonal. NaN entries correspond to the 
    case where the given regret was not attained after $T=20$ evaluations.}
\label{fig:band_ucb_vs_ucb2}
\begin{tabular}{llrrrrr}
\toprule
   &       &  2.00 &  1.55 &  1.10 &  0.65 &  0.20 \\
band\_size & band\_corr &       &       &       &       &       \\
\midrule
0  &  0.00 &   1.0 &   1.0 &  1.00 &   NaN &   NaN \\
2  & -0.20 &   1.0 &   1.0 &  0.86 &  0.87 &   NaN \\
3  &  0.20 &   1.0 &   1.0 &  0.87 &  0.89 &   NaN \\
5  & -0.10 &   1.0 &   1.0 &  1.00 &  0.88 &   NaN \\
   &  0.20 &   1.0 &   1.0 &  1.00 &  0.88 &   NaN \\
10 &  0.10 &   1.0 &   1.0 &  1.00 &  0.83 &   NaN \\
40 &  0.05 &   1.0 &   1.0 &  1.00 &  1.00 &   NaN \\
\bottomrule
\end{tabular}
\end{table}

\subsection{Randomly-Sampled Covariance Matrices}
Finally, we sample covariances (of size $N=200$) randomly from an inverse Wishart distribution (with $400$ degrees of freedom and identity scale matrix). We report the results in \ref{fig:wishart_ei_vs_ei2} for EI vs EI2 and in \ref{fig:wishart_ucb_vs_ucb2} for UCB vs UCB2. As in the previous experiments, we see that $0.5t_{ei2}\le t_{ei}\le t_{ei2}$ (and the equivalent for UCB).
\begin{table}
\centering
\caption{Fraction $\frac{t_{ei}(\text{wishart\_seed})}
    {t_{ei2}(\text{wishart\_seed})}$ for covariance matrices 
    drawn from a Wishart distribution. NaN entries correspond to the 
    case where the given regret was not attained after $T=30$ evaluations.}
\label{fig:wishart_ei_vs_ei2}
\begin{tabular}{lrrrrr}
\toprule
{} &  0.20 &  0.15 &  0.10 &  0.05 &  0.00 \\
wishart\_seed &       &       &       &       &       \\
\midrule
1            &   1.0 &  0.67 &  0.83 &  0.78 &   NaN \\
2            &   1.0 &  1.00 &  0.83 &  0.87 &   NaN \\
3            &   1.0 &  1.00 &  0.83 &  0.83 &   NaN \\
4            &   1.0 &  1.00 &  1.00 &  0.82 &   NaN \\
5            &   1.0 &  1.00 &  1.00 &  0.82 &   NaN \\
\bottomrule
\end{tabular}
\end{table}

\begin{table}
\centering
\caption{Fraction $\frac{t_{ucb}(\text{wishart\_seed})}
    {t_{ucb2}(\text{wishart\_seed})}$ for covariance matrices 
    drawn from a Wishart distribution. NaN entries correspond to the 
    case where the given regret was not attained after $T=30$ evaluations.}
\label{fig:wishart_ucb_vs_ucb2}
\begin{tabular}{lrrrrr}
\toprule
{} &  0.20 &  0.15 &  0.10 &  0.05 &  0.00 \\
wishart\_seed &       &       &       &       &       \\
\midrule
1            &   1.0 &   1.0 &  0.83 &  0.82 &   NaN \\
2            &   1.0 &   1.0 &  0.83 &  0.87 &   NaN \\
3            &   1.0 &   1.0 &  0.83 &  0.88 &   NaN \\
4            &   1.0 &   1.0 &  1.00 &  0.87 &   NaN \\
5            &   1.0 &   1.0 &  1.00 &  0.81 &   NaN \\
\bottomrule
\end{tabular}
\end{table}

\newpage

\section{Proof of \ref{theorem:main}}

\label{sec:proof}

In this section we prove \ref{theorem:main}. As we have seen in the
previous section, our problem is not adaptively submodular. Nevertheless,
the following proof is heavily inspired by the proof in \cite{golovin11adaptive}.
We derive a less strict condition than adaptive submodularity which
is applicable to our problem and implies that we converge exponentially
to the optimum$\cdot\beta$: \begin{lemma} \label{lemma:bo_main_inequality}For
	any problem of the type defined in \ref{sec:po_problem_statement},
	we have for any $\alpha,\beta>0$ 
	\begin{align}
		\beta\mathbb{E}\left[\hat{\check{F}}\right]-\mathbb{E}\left[\hat{\check{Y}}_{t}\right] & \le\alpha\left(\mathbb{E}\left[\hat{\check{Y}}_{t+1}\right]-\mathbb{E}\left[\hat{\check{Y}}_{t}\right]\right)\quad\forall t\in\{1:T-1\}\label{eq:bo_main_inequality}\\
		& \Downarrow\nonumber \\
		(1-e^{-\frac{T-1}{\alpha}})\beta\mathbb{E}\left[\hat{\check{F}}\right] & \le\mathbb{E}\left[\hat{\check{Y}}_{T}\right].
	\end{align}
	i.e. the first inequality implies the second inequality.\end{lemma}
\begin{proof} The proof is closely related to the adaptive submodularity
	proof by \citet{golovin11adaptive}. Defining $\delta_{t}:=\beta\mathbb{E}\left[\hat{\check{F}}\right]-\mathbb{E}\left[\hat{\check{Y}}_{t}\right]\forall t\in\{1:T\}$,
	we can rewrite the first inequality as
	
	\begin{align*}
		\delta_{t} & \le\alpha(\delta_{t}-\delta_{t+1})\quad\forall t\in\{1:T-1\}\\
		\delta_{t+1} & \le(1-\frac{1}{\alpha})\delta_{t}\quad\forall t\in\{1:T-1\}
	\end{align*}
	Since the function $e^{x}$ is convex, we have $e^{x}\ge1+x\quad\forall x$.
	Using this fact, we obtain the inequality
	
	\begin{align*}
		\delta_{t+1} & \le e^{-\frac{1}{\alpha}}\delta_{t}\quad\forall t\in\{1:T-1\}\\
		\delta_{T} & \le e^{-\frac{T-1}{\alpha}}\delta_{1}
	\end{align*}
	Now we can substitute $\delta_{T}=\beta\mathbb{E}\left[\hat{\check{F}}\right]-\mathbb{E}\left[\hat{\check{Y}}_{T}\right]$
	and $\delta_{1}=\beta\mathbb{E}\left[\hat{\check{F}}\right]-\mathbb{E}\left[\hat{\check{Y}}_{1}\right]=\beta\mathbb{E}\left[\hat{\check{F}}\right]$
	(since $\hat{\check{Y}}_{1}=0)$:
	
	\begin{align*}
		\beta\mathbb{E}\left[\hat{\check{F}}\right]-\mathbb{E}\left[\hat{\check{Y}}_{T}\right] & \le e^{-\frac{T-1}{\alpha}}\beta\mathbb{E}\left[\hat{\check{F}}\right]\\
		(1-e^{-\frac{T-1}{\alpha}})\beta\mathbb{E}\left[\hat{\check{F}}\right] & \le\mathbb{E}\left[\hat{\check{Y}}_{T}\right].
	\end{align*}
	
\end{proof} % Anteprima del sorgente dal paragrafo 98 al 141
Hence, if for some $\alpha,\beta>0$ we can show that \ref{eq:bo_main_inequality}
holds, \ref{lemma:bo_main_inequality} yields a lower bound on the
expected utility.

\subsection{Specialization for the Extremization Problem}

\begin{lemma} \label{lemma:extremization_main_inequality}For any
	problem of the type defined in \ref{sec:po_problem_statement} we
	have for any $\alpha,\beta>0$ 
	\begin{align}
		\mathbb{E}\left[\beta\hat{\check{F}}-\hat{\check{Y}}_{t}|m_{t},c_{t},\hat{y}_{t},\check{y}_{t}\right] & \le\alpha\mathbb{E}\left[\hat{\check{Y}}_{t+1}-\hat{\check{Y}}_{t}|m_{t},c_{t},\hat{y}_{t},\check{y}_{t}\right]\label{eq:extremization_main_inequality}\\
		& \quad\quad\quad\quad\forall t\in\{1:T-1\},m_{t}\in\mathbb{R}^{N},c_{t}\in\mathbb{S}_{+}^{N},\hat{y}_{t}\ge\check{y}_{t}\in\mathbb{R}\nonumber \\
		& \Downarrow\nonumber \\
		(1-e^{-\frac{T-1}{\alpha}})\beta\mathbb{E}\left[\hat{\check{F}}\right] & \le\mathbb{E}\left[\hat{\check{Y}}\right].
	\end{align}
\end{lemma} \begin{proof} It is easy to see that the implication
	\begin{align}
		\mathbb{E}\left[\beta\hat{\check{F}}-\hat{\check{Y}}_{t}|m_{t},c_{t},\hat{y}_{t},\check{y}_{t}\right] & \le\alpha\mathbb{E}\left[\hat{\check{Y}}_{t+1}-\hat{\check{Y}}_{t}|m_{t},c_{t},\hat{y}_{t},\check{y}_{t}\right]\\
		& \quad\quad\quad\quad\forall t\in[T-1],m_{t}\in\mathbb{R}^{N},c_{t}\in\mathbb{S}_{+}^{N},\hat{y}_{t}\ge\check{y}_{t}\in\mathbb{R}\nonumber \\
		& \Downarrow\nonumber \\
		\mathbb{E}\left[\beta\hat{\check{F}}-\hat{\check{Y}}_{t}\right] & \le\alpha\mathbb{E}\left[\hat{\check{Y}}_{t+1}-\hat{\check{Y}}_{t}\right]\quad\forall t\in\{1:T-1\}.\label{eq:intermediate1}
	\end{align}
	holds, since taking the expectation with respect to $M_{t},C_{t},\hat{Y}_{t},\check{Y}_{t}$
	on both sides of the first line yields the second line. Since \ref{eq:intermediate1}
	is identical to \ref{lemma:bo_main_inequality}, the desired implication
	follows from these two implications. \end{proof} To prove that \ref{eq:extremization_main_inequality}
holds, we will derive a lower bound for the right-hand side and an
upper bound for the left-hand side.

\subsection{Lower Bound for the Right-Hand Side}

\begin{lemma} \label{lemma:bound_rhs} For any instance $(N,T,\mu,\Sigma)$
	of the problem defined in \ref{sec:po_problem_statement}, if we follow
	either the EI2 (\ref{def:ei_extremization}) or the UCB2 (\ref{def:ucb_extremization})
	strategy, we have 
	\begin{align}
		\mathbb{E} & \left[\hat{\check{Y}}_{t+1}-\hat{\check{Y}}_{t}|m_{t},c_{t},\hat{y}_{t},\check{y}_{t}\right]\label{eq:sdfasdfasdf}\\
		& \ge\max\left\{ \op{ei}\left(\hat{y}_{t}\bigg|m_{t}^{n_{ucb}},\sqrt{c_{t}^{n_{ucb}n_{ucb}}}\right),\op{ei}\left(-\check{y}_{t}\bigg|-m_{t}^{n_{ucb}},\sqrt{c_{t}^{n_{ucb}n_{ucb}}}\right)\right\} \nonumber 
	\end{align}
	with $\op{ei}$ as defined in \ref{def:ei} and 
	\begin{equation}
	n_{ucb}:=\op{argmax}_{n\in[N]}\max\left\{ -\hat{y}_{t}+m_{t}^{n}+\sqrt{c_{t}^{nn}2\log N},\check{y}_{t}-m_{t}^{n}+\sqrt{c_{t}^{nn}2\log N}\right\} 
	\end{equation}
	for any $t\in\{1:T-1\},m_{t}\in\mathbb{R}^{N},c_{t}\in\mathbb{S}_{+}^{N},\hat{y}_{t}\ge\check{y}_{t}\in\mathbb{R}$.
\end{lemma}

\begin{proof} Developing the expectation on the left hand side of
	\ref{eq:extremization_main_inequality} we have 
	\begin{align}
		\mathbb{E} & \left[\hat{\check{Y}}_{t+1}-\hat{\check{Y}}_{t}|m_{t},c_{t},\hat{y}_{t},\check{y}_{t}\right]\\
		& =\mathbb{E}\left[\max\left\{ F_{A_{t+1}}-\hat{y}_{t},0\right\} +\max\left\{ -F_{A_{t+1}}+\check{y}_{t},0\right\} \bigg|m_{t},c_{t},\hat{y}_{t},\check{y}_{t}\right]\\
		& =\mathbb{E}\left[\op{ei}\left(\hat{y}_{t}\bigg|m_{t}^{A_{t+1}},\sqrt{c_{t}^{A_{t+1}A_{t+1}}}\right)+\op{ei}\left(-\check{y}_{t}\bigg|-m_{t}^{A_{t+1}},\sqrt{c_{t}^{A_{t+1}A_{t+1}}}\right)\bigg|m_{t},c_{t},\hat{y}_{t},\check{y}_{t}\right]\\
		& \ge\mathbb{E}\left[\max\left\{ \op{ei}\left(\hat{y}_{t}\bigg|m_{t}^{A_{t+1}},\sqrt{c_{t}^{A_{t+1}A_{t+1}}}\right),\op{ei}\left(-\check{y}_{t}\bigg|-m_{t}^{A_{t+1}},\sqrt{c_{t}^{A_{t+1}A_{t+1}}}\right)\right\} \bigg|m_{t},c_{t},\hat{y}_{t},\check{y}_{t}\right]\label{bla}
	\end{align}
	where we have used \ref{def:ei}, and the inequality follows from
	the fact that the expected improvement ($\op{ei}$) is always $\ge0$.
	The action $A_{t+1}$ is a function of $M_{t},C_{t},\hat{Y}_{t},\check{Y}_{t}$.
	If we follow the EI2 strategy (\ref{def:ei_extremization}), we have
	
	\begin{align}
		\ref{bla} & =\max_{n\in[N]}\max\left\{ \op{ei}\left(\hat{y}_{t}|m_{t}^{n},\sqrt{c_{t}^{nn}}\right),\op{ei}\left(-\check{y}_{t}|-m_{t}^{n},\sqrt{c_{t}^{nn}}\right)\right\} \label{eq:rhs_intermediate_ei}
	\end{align}
	and if we follow the UCB2 (\ref{def:ucb_extremization}) strategy,
	we have 
	\begin{equation}
	\ref{bla}=\max\left\{ \op{ei}\left(\hat{y}_{t}\bigg|m_{t}^{n_{ucb}},\sqrt{c_{t}^{n_{ucb}n_{ucb}}}\right),\op{ei}\left(-\check{y}_{t}\bigg|-m_{t}^{n_{ucb}},\sqrt{c_{t}^{n_{ucb}n_{ucb}}}\right)\right\} .\label{eq:rhs_intermediate_ucb}
	\end{equation}
	Clearly we have $\ref{eq:rhs_intermediate_ei}\ge\ref{eq:rhs_intermediate_ucb}$,
	hence for both strategies it holds that $\ref{bla}\ge\ref{eq:rhs_intermediate_ucb}$
	which concludes the proof.
	
\end{proof}

\subsection{Upper Bound for the Left-Hand Side}

In analogy with \ref{def:ei}, we define

\begin{definition}[Multivariate Expected Improvement]\label{def:nei}
	For a family of jointly Gaussian distributed RVs $(F_{n})_{n\in[N]}$
	with mean $m\in\mathbb{R}^{N}$ and covariance $c\in\mathbb{S}_{+}^{N}$
	and a threshold $\tau\in\mathbb{R}$, we define the multivariate expected
	improvement as 
	\begin{align}
		\op{mei}(\tau|m,c): & =\mathbb{E}\left[\max\left\{ \max_{n\in[N]}F_{n}-\tau,0\right\} \right]\\
		& =\int_{\mathbb{R}^{N}}\max\left\{ \max_{n\in[N]}f_{n}-\tau,0\right\} \mathcal{N}(f|m,c)df.
	\end{align}
\end{definition}

\begin{lemma} \label{lemma:bound_lhs} For any instance $(N,T,\mu,\Sigma)$
	of the problem defined in \ref{sec:po_problem_statement}, if we follow
	either the EI2 (\ref{def:ei_extremization}) or the UCB2 (\ref{def:ucb_extremization})
	strategy, we have for any $0<\beta\le1$ 
	\begin{align}
		\mathbb{E} & \left[\beta\hat{\check{F}}-\hat{\check{Y}}_{t}|m_{t},c_{t},\hat{y}_{t},\check{y}_{t}\right]\label{eq:sdfasdfasdf-1}\\
		& \le\beta2\max\left\{ \op{mei}(0|m_{t}-\hat{y}_{t},c_{t}),\op{mei}(0|\check{y}_{t}-m_{t},c_{t})\right\} +(1-\beta)(-\hat{y}_{t}+\check{y}_{t})\nonumber 
	\end{align}
	for any $t\in\{1:T-1\},m_{t}\in\mathbb{R}^{N},c_{t}\in\mathbb{S}_{+}^{N},\hat{y}_{t}\ge\check{y}_{t}\in\mathbb{R}$.
	
\end{lemma}

\begin{proof}We have
	\begin{align}
		\mathbb{E} & \left[\beta\hat{\check{F}}-\hat{\check{Y}}_{t}|m_{t},c_{t},\hat{y}_{t},\check{y}_{t}\right]\\
		& =\beta\mathbb{E}\left[\hat{F}-\check{F}-\hat{y}_{t}+\check{y}_{t}|m_{t},c_{t}\right]+(1-\beta)(-\hat{y}_{t}+\check{y}_{t})\\
		& =\beta\mathbb{E}\left[\hat{F}-\hat{y}_{t}|m_{t},c_{t}\right]+\beta\mathbb{E}\left[-\check{F}+\check{y}_{t}|m_{t},c_{t}\right]+(1-\beta)(-\hat{y}_{t}+\check{y}_{t})\\
		& \le\beta\mathbb{E}\left[\max\{\hat{F}-\hat{y}_{t},0\}|m_{t},c_{t}\right]+\beta\mathbb{E}\left[\max\{-\check{F}+\check{y}_{t},0\}|m_{t},c_{t}\right]+(1-\beta)(-\hat{y}_{t}+\check{y}_{t})\\
		& =\beta\left(\op{mei}(0|m_{t}-\hat{y}_{t},c_{t})+\op{mei}(0|\check{y}_{t}-m_{t},c_{t})\right)+(1-\beta)(-\hat{y}_{t}+\check{y}_{t})\\
		& \le\beta2\max\left\{ \op{mei}(0|m_{t}-\hat{y}_{t},c_{t}),\op{mei}(0|\check{y}_{t}-m_{t},c_{t})\right\} +(1-\beta)(-\hat{y}_{t}+\check{y}_{t})
	\end{align}
	where by $m_{t}-\hat{y}_{t}$ we mean that the scalar $\hat{y}_{t}$
	is subtracted from each element of the vector $m_{t}$.
	
\end{proof}

\subsection{Upper Bound on the Regret}

\begin{theorem}\label{theorem:bound_alpha} For any instance $(N,T,\mu,\Sigma)$
	of the problem defined in \ref{sec:po_problem_statement}, if we follow
	either the EI2 (\ref{def:ei_extremization}) or the UCB2 (\ref{def:ucb_extremization})
	strategy, we have for any $\alpha>0$ and any $0<\beta\le1-\frac{1}{\sqrt{2\pi}\left(2\log N\right)^{3/2}}$
	\begin{align}
		\max_{x}\left(2\frac{\beta\left(\sqrt{2\log N}+\frac{1}{2\log N\sqrt{2\pi}}\right)-x}{\op{ei}\left(x\right)}\right) & \le\alpha\label{eq:bound_alpha_condition}\\
		& \Downarrow\nonumber \\
		(1-e^{-\frac{T-1}{\alpha}})\beta\mathbb{E}\left[\hat{\check{F}}\right] & \le\mathbb{E}\left[\hat{\check{Y}}\right]\label{eq:bound_alpha_implication}
	\end{align}
	i.e. the first line implies the second. \end{theorem} % Anteprima del sorgente dal paragrafo 78 al 81

\begin{proof} According to \ref{lemma:extremization_main_inequality},
	we have 
	\begin{align}
		& \ref{eq:bound_alpha_implication}\nonumber \\
		& \Uparrow\nonumber \\
		& \frac{\mathbb{E}\left[\beta\hat{\check{F}}-\hat{\check{Y}}_{t}|m_{t},c_{t},\hat{y}_{t},\check{y}_{t}\right]}{\mathbb{E}\left[\hat{\check{Y}}_{t+1}-\hat{\check{Y}}_{t}|m_{t},c_{t},\hat{y}_{t},\check{y}_{t}\right]}\le\alpha\nonumber \\
		& \quad\quad\forall t\in\{1:T-1\},m_{t}\in\mathbb{R}^{N},c_{t}\in\mathbb{S}_{+}^{N},\hat{y}_{t}\ge\check{y}_{t}\in\mathbb{R}.\label{eq:regretChain-1}
	\end{align}
	In the following, we will find a simpler expression which implies
	\ref{eq:regretChain-1} and therefore \ref{eq:bound_alpha_implication}.
	Then we will simplify the new expression further, until we finally
	arrive at \ref{eq:bound_alpha_condition} through an unbroken chain
	of implications.
	
	Using the lower bound from \ref{lemma:bound_rhs} and the upper bound
	from \ref{lemma:bound_lhs} we can write 
	\begin{align}
		& \ref{eq:regretChain-1}\nonumber \\
		& \Uparrow\nonumber \\
		& \frac{\beta2\max\left\{ \op{mei}(0|m-\hat{y},c),\op{mei}(0|\check{y}-m,c)\right\} +(1-\beta)(-\hat{y}+\check{y})}{\max\left\{ \op{ei}\left(\hat{y}\bigg|m^{n_{ucb}},\sqrt{c^{n_{ucb}n_{ucb}}}\right),\op{ei}\left(-\check{y}\bigg|-m^{n_{ucb}},\sqrt{c^{n_{ucb}n_{ucb}}}\right)\right\} }\le\alpha\nonumber \\
		& \quad\quad\forall m\in\mathbb{R}^{N},c\in\mathbb{S}_{+}^{N},\hat{y}\ge\check{y}\in\mathbb{R},\label{eq:regretChain0}\\
		& \quad\quad\quad\quad n_{ucb}=\op{argmax}_{n\in[N]}\max\left\{ -\hat{y}+m^{n}+\sqrt{c^{nn}2\log N},\check{y}-m^{n}+\sqrt{c^{nn}2\log N}\right\} \nonumber 
	\end{align}
	where we have dropped the time indices, since they are irrelevant
	here. It is easy to see that all terms in \ref{eq:regretChain0} are
	invariant to a common shift in $m,\hat{y}$ and $\check{y}$. Hence,
	we can impose a constraint on these variables, without changing the
	condition. We choose the constraint $\hat{y}=-\check{y}$ to simplify
	the expression 
	\begin{align}
		& \ref{eq:regretChain0}\nonumber \\
		& \Updownarrow\nonumber \\
		& 2\frac{\beta\max\left\{ \op{mei}(0|m-\hat{y},c),\op{mei}(0|-\hat{y}-m,c)\right\} -(1-\beta)\hat{y}}{\max\left\{ \op{ei}\left(\hat{y}\bigg|m^{n_{ucb}},\sqrt{c^{n_{ucb}n_{ucb}}}\right),\op{ei}\left(\hat{y}\bigg|-m^{n_{ucb}},\sqrt{c^{n_{ucb}n_{ucb}}}\right)\right\} }\le\alpha\nonumber \\
		& \quad\quad\forall m\in\mathbb{R}^{N},c\in\mathbb{S}_{+}^{N},\hat{y}\in\mathbb{R}_{\ge0},\label{eq:regretChain1}\\
		& \quad\quad\quad\quad n_{ucb}=\op{argmax}_{n\in[N]}\max\left\{ -\hat{y}+m^{n}+\sqrt{c^{nn}2\log N},-\hat{y}-m^{n}+\sqrt{c^{nn}2\log N}\right\} .\nonumber 
	\end{align}
	Clearly, the denominator is invariant with respect to any sign flips
	in the elements of $m$. The numerator is maximized if all elements
	of $m$ have the same sign, no matter if positive or negative. Hence,
	we can restrict the above conditions to $m$ with positive entries,
	which means in all maximum operators the left term is active 
	\begin{align}
		& \ref{eq:regretChain1}\nonumber \\
		& \Updownarrow\nonumber \\
		& 2\frac{\beta\op{mei}(0|m-\hat{y},c)-(1-\beta)\hat{y}}{\op{ei}\left(\hat{y}\bigg|m^{n_{ucb}},\sqrt{c^{n_{ucb}n_{ucb}}}\right)}\le\alpha\nonumber \\
		& \quad\quad\forall m\in\mathbb{R}_{\ge0}^{N},c\in\mathbb{S}_{+}^{N},\hat{y}\in\mathbb{R}_{\ge0},n_{ucb}=\op{argmax}_{n\in[N]}\left(m^{n}+\sqrt{c^{nn}2\log N}\right).\label{eq:regretChain2}
	\end{align}
	Inserting the bound from \ref{lemma:bound_nei} we have 
	\begin{align}
		& \ref{eq:regretChain2}\nonumber \\
		& \Uparrow\nonumber \\
		& 2\frac{\beta\left(\max\left\{ \max_{n\in[N]}(m^{n}-\hat{y}+\sqrt{c^{nn}2\log N}),0\right\} +\frac{\max_{n\in[N]}\sqrt{c^{nn}}}{2\sqrt{2\pi}\log\left(N\right)}\right)-(1-\beta)\hat{y}}{\op{ei}\left(\hat{y}\bigg|m^{n_{ucb}},\sqrt{c^{n_{ucb}n_{ucb}}}\right)}\le\alpha\nonumber \\
		& \quad\quad\forall m\in\mathbb{R}_{\ge0}^{N},c\in\mathbb{S}_{+}^{N},\le\alpha,\hat{y}\in\mathbb{R}_{\ge0},n_{ucb}=\op{argmax}_{n\in[N]}\left(m^{n}+\sqrt{c^{nn}2\log N}\right)\\
		& \Updownarrow\nonumber \\
		& 2\frac{\beta\left(\max\left\{ m^{n_{ucb}}-\hat{y}+\sqrt{c^{n_{ucb}n_{ucb}}2\log N},0\right\} +\frac{\max_{n\in[N]}\sqrt{c^{nn}}}{2\sqrt{2\pi}\log\left(N\right)}\right)-(1-\beta)\hat{y}}{\op{ei}\left(\hat{y}\bigg|m^{n_{ucb}},\sqrt{c^{n_{ucb}n_{ucb}}}\right)}\le\alpha\nonumber \\
		& \quad\quad\forall m\in\mathbb{R}_{\ge0}^{N},c\in\mathbb{S}_{+}^{N},\le\alpha,\hat{y}\in\mathbb{R}_{\ge0},n_{ucb}=\op{argmax}_{n\in[N]}\left(m^{n}+\sqrt{c^{nn}2\log N}\right).\label{eq:regretChain3}
	\end{align}
	It holds for any $n\in[N]$ that $m^{n}\ge0$ and 
	\[
	m^{n_{ucb}}+\sqrt{c^{n_{ucb}n_{ucb}}}\sqrt{2\log N}\ge m^{n}+\sqrt{c^{nn}2\log N},
	\]
	from which it follows that 
	\begin{align}
		\frac{m^{n_{ucb}}}{\sqrt{2\log N}}+\sqrt{c^{n_{ucb}n_{ucb}}} & \ge\sqrt{c^{nn}}.
	\end{align}
	Using this fact we can write 
	\begin{align}
		& \ref{eq:regretChain3}\nonumber \\
		& \Uparrow\nonumber \\
		& 2\frac{\beta\left(\max\{m^{n_{ucb}}-\hat{y}+\sqrt{c^{n_{ucb}n_{ucb}}}\sqrt{2\log N},0\}+\frac{\sqrt{c^{n_{ucb}n_{ucb}}}}{2\log N\sqrt{2\pi}}+\frac{m^{n_{ucb}}}{\sqrt{2\pi}\left(2\log N\right)^{3/2}}\right)-(1-\beta)\hat{y}}{\op{ei}\left(\hat{y}\bigg|m^{n_{ucb}},\sqrt{c^{n_{ucb}n_{ucb}}}\right)}\le\alpha\nonumber \\
		& \quad\quad\forall m\in\mathbb{R}_{\ge0}^{N},c\in\mathbb{S}_{+}^{N},\le\alpha,\hat{y}\in\mathbb{R}_{\ge0},n_{ucb}=\op{argmax}_{n\in[N]}\left(m^{n}+\sqrt{c^{nn}2\log N}\right).\label{eq:regretChain4}
	\end{align}
	For any $\beta\le1-\frac{1}{1+\sqrt{2\pi}\left(2\log N\right)^{3/2}}\le1-\frac{1}{\sqrt{2\pi}\left(2\log N\right)^{3/2}}$
	we have 
	\begin{align}
		& \ref{eq:regretChain4}\nonumber \\
		& \Uparrow\nonumber \\
		& 2\frac{\beta\left(\max\{m^{n_{ucb}}-\hat{y}+\sqrt{c^{n_{ucb}n_{ucb}}}\sqrt{2\log N},0\}+\frac{\sqrt{c^{n_{ucb}n_{ucb}}}}{2\log N\sqrt{2\pi}}\right)-(1-\beta)(\hat{y}-m^{n_{ucb}})}{\op{ei}\left(\hat{y}\bigg|m^{n_{ucb}},\sqrt{c^{n_{ucb}n_{ucb}}}\right)}\le\alpha\nonumber \\
		& \quad\quad\forall m\in\mathbb{R}_{\ge0}^{N},c\in\mathbb{S}_{+}^{N},\le\alpha,\hat{y}\in\mathbb{R}_{\ge0},n_{ucb}=\op{argmax}_{n\in[N]}\left(m^{n}+\sqrt{c^{nn}2\log N}\right).\label{eq:regretChain5}
	\end{align}
	According to \ref{def:ei}, we have 
	\[
	\op{ei}\left(\hat{y}\bigg|\mu,\sqrt{c^{n_{ucb}n_{ucb}}}\right)=\sqrt{c^{n_{ucb}n_{ucb}}}\op{ei}\left(\frac{\hat{y}-m^{n_{ucb}}}{\sqrt{c^{n_{ucb}n_{ucb}}}}\right).
	\]
	Using this fact, we obtain 
	\begin{align}
		& \ref{eq:regretChain5}\nonumber \\
		& \Updownarrow\nonumber \\
		& 2\frac{\beta\left(\max\left\{ \frac{m^{n_{ucb}}-\hat{y}}{\sqrt{c^{n_{ucb}n_{ucb}}}}+\sqrt{2\log N},0\right\} +\frac{1}{2\log N\sqrt{2\pi}}\right)-(1-\beta)\frac{\hat{y}-m^{n_{ucb}}}{\sqrt{c^{n_{ucb}n_{ucb}}}}}{\op{ei}\left(\frac{\hat{y}-m^{n_{ucb}}}{\sqrt{c^{n_{ucb}n_{ucb}}}}\right)}\le\alpha\nonumber \\
		& \quad\quad\forall m\in\mathbb{R}_{\ge0}^{N},c\in\mathbb{S}_{+}^{N},\le\alpha,\hat{y}\in\mathbb{R}_{\ge0},n_{ucb}=\op{argmax}_{n\in[N]}\left(m^{n}+\sqrt{c^{nn}2\log N}\right).\label{eq:regretChain6}
	\end{align}
	Defining $x:=\frac{\hat{y}-m^{n_{ucb}}}{\sqrt{c^{n_{ucb}n_{ucb}}}}$
	we can simplify this condition as 
	\begin{align}
		& \ref{eq:regretChain6}\nonumber \\
		& \Updownarrow\nonumber \\
		& 2\frac{\beta\left(\max\left\{ \sqrt{2\log N}-x,0\right\} +\frac{1}{2\log N\sqrt{2\pi}}\right)-(1-\beta)x}{\op{ei}\left(x\right)}\le\alpha\quad\quad\forall x\in\mathbb{R}.\label{eq:regretChain7}
	\end{align}
	For any $x\ge\sqrt{2\log N}$ the numerator of the left hand side
	is negative 
	\begin{align}
		\beta\left(\frac{1}{2\log N\sqrt{2\pi}}\right)-(1-\beta)x & \le\beta\left(\frac{1}{2\log N\sqrt{2\pi}}\right)-(1-\beta)\sqrt{2\log N}\\
		& \le(1-\frac{1}{\sqrt{2\pi}\left(2\log N\right)^{3/2}})\left(\frac{1}{2\log N\sqrt{2\pi}}\right)-\frac{1}{\sqrt{2\pi}\left(2\log N\right)}\\
		& =-\frac{1}{2\log N\sqrt{2\pi}}\left(\frac{1}{2\log N\sqrt{2\pi}}\right)\\
		& \le0
	\end{align}
	and since $\alpha$ and the denominator are both positive, \ref{eq:regretChain7}
	is satisfied. Hence, we only need to consider the case where $x\le\sqrt{2\log N}$
	and can therefore write 
	\begin{align}
		& \ref{eq:regretChain7}\nonumber \\
		& \Updownarrow\nonumber \\
		& 2\frac{\beta\left(\sqrt{2\log N}+\frac{1}{2\log N\sqrt{2\pi}}\right)-x}{\op{ei}\left(x\right)}\le\alpha\quad\forall x\le\sqrt{2\log N}.
	\end{align}
	From this chain of implications and \ref{lemma:extremization_main_inequality}
	the result of \ref{theorem:bound_alpha} follows.\end{proof}

% Anteprima del sorgente dal paragrafo 124 al 128
Finally, using the previous results, we can obtain the desired bound
on the regret (\ref{theorem:main}), which we restate here for convenience:

\maintheorem*

\begin{proof} We can rewrite \ref{theorem:bound_alpha}
	as 
	
	\begin{align}
		\left(1-e^{-\frac{T}{\max_{x}\left(2\frac{\beta\left(\sqrt{2\log N}+\frac{1}{2\log N\sqrt{2\pi}}\right)-x}{\op{ei}\left(x\right)}\right)}}\right)\beta & \le\frac{\mathbb{E}\left[\hat{\check{Y}}\right]}{\mathbb{E}\left[\hat{\check{F}}\right]}\\
		\forall N\ge T\ge500,\mu,\Sigma, & 0<\beta\le1-\frac{1}{\sqrt{2\pi}\left(2\log N\right)^{3/2}}\nonumber 
	\end{align}
%	\comment{should it be T -1 instead of T here?}
	where we have restricted the inequality to $N\ge T\ge500$, a condition
	we need later . What is left to be done is to simplify this bound,
	such that it becomes interpretable. We can rewrite the above as 
	\begin{align}
		\left(1-e^{-\frac{T}{\max_{x}\left(2\frac{\beta\left(\frac{1}{\sqrt{2\pi}a^{2}}+a\right)-x}{\op{ei}\left(x\right)}\right)}}\right)\beta & \le\frac{\mathbb{E}\left[\hat{\check{Y}}\right]}{\mathbb{E}\left[\hat{\check{F}}\right]}\label{eq:9990}\\
		\forall N\ge T\ge500,\mu,\Sigma, & 0<\beta\le1-\frac{1}{\sqrt{2\pi}a^{3}},a=\sqrt{2\log N}.\nonumber 
	\end{align}
	
	\subsection*{Choosing a $\beta$}
	
	We choose 
	\begin{align}
		\beta & \text{=}\frac{b-\frac{\text{ei}(b)}{\text{ei}'(b)}}{\frac{1}{\sqrt{2\pi}a^{2}}+a}\label{eq:chosen_beta}
	\end{align}
	where $\text{ei}'$ is the derivative of $\text{ei}$, and 
	\begin{equation}
	b=\sqrt{\sqrt{2\log T}^{2}-2\log\left(3\ 2^{-3/2}\sqrt{2\log T}^{3}\right)}.\label{eq:def_b}
	\end{equation}
	Before we can continue, we need to show that this $\beta$ satisfies
	the condition in \ref{eq:9990}. First of all, from \ref{eq:def_b}
	and the conditions in \ref{eq:9990} we can derive some relations
	which will be useful later on 
	\begin{align}
		2.17\le b & \le\sqrt{a^{2}-2\log\left(3\ 2^{-3/2}a^{3}\right)}\le a.\label{eq:relations_ab}
	\end{align}
	It is easy to see that $0<\beta$ holds, it remains to be shown that
	\begin{align}
		\beta & \le1-\frac{1}{\sqrt{2\pi}a^{3}}.
	\end{align}
	Inserting the definition of $\beta$ \ref{eq:chosen_beta} and simplifying
	we have 
	\begin{align}
		\frac{b-\frac{\text{ei}(b)}{\text{ei}'(b)}}{\frac{1}{\sqrt{2\pi}a^{2}}+a} & \le1-\frac{1}{\sqrt{2\pi}a^{3}}\\
		b-\frac{\text{ei}(b)}{\text{ei}'(b)} & \leq a-\frac{1}{2\pi a^{5}}.
	\end{align}
	Using \ref{def:ei} and the lower bound in \ref{lemma:bounds_on_normal_ccdf},
	we obtain a sufficient condition 
	\begin{equation}
	\frac{b^{3}}{b^{2}-1}\leq a-\frac{1}{2\pi a^{5}}.
	\end{equation}
	Since we have $b\le a$ (see \ref{eq:relations_ab}) we obtain the
	sufficient condition 
	\begin{equation}
	\frac{1}{2\pi b^{5}}+\frac{b^{3}}{b^{2}-1}\leq a.
	\end{equation}
	From $2.17\le b$ (see \ref{eq:relations_ab}) it follows that $2\pi b^{5}\ge b^{2}-1$,
	and hence we can further simplify to obtain a sufficient condition
	\begin{equation}
	b+\frac{1}{b-1}\leq a.
	\end{equation}
	Given \ref{eq:relations_ab} it is easy to see that both sides are
	positive, hence we can square each side to obtain 
	\begin{equation}
	b^{2}+\frac{2b}{b-1}+\frac{1}{(b-1)^{2}}\leq a^{2}.\label{eq:beta_condition_final}
	\end{equation}
	Now we will show that this condition is satisfied by \ref{eq:def_b}.
	We have from \ref{eq:relations_ab} 
	\begin{align}
		b & \le\sqrt{a^{2}-2\log\left(3\ 2^{-3/2}a^{3}\right)}\\
		2\log\left(3\ 2^{-3/2}a^{3}\right)+b^{2} & \leq a^{2}
	\end{align}
	and since $b\le a$, we have 
	\begin{equation}
	2\log\left(3\ 2^{-3/2}b^{3}\right)+b^{2}\leq a^{2}.
	\end{equation}
	This implies \ref{eq:beta_condition_final}, to see this we use the
	above inequality to bound $a^{2}$ in \ref{eq:beta_condition_final}
	\begin{align}
		b^{2}+\frac{2b}{b-1}+\frac{1}{(b-1)^{2}} & \le2\log\left(3\ 2^{-3/2}b^{3}\right)+b^{2}\\
		\frac{2b}{b-1}+\frac{1}{(b-1)^{2}} & \le2\log\left(3\ 2^{-3/2}b^{3}\right).
	\end{align}
	It is easy to verify that for any $b\ge2.17$, the left-hand side
	is decreasing and the right-hand side is increasing. Hence, it is
	sufficient to show that it holds for $b=2.17$ which is easily done
	by evaluating at that value. This concludes the proof that our choice
	of $\beta$ \ref{eq:chosen_beta} satisfies the conditions in \ref{eq:9990}.
	We can now insert this value to obtain
	
	\begin{align}
		\left(1-e^{-\frac{T}{\max_{x}\left(2\frac{b-\frac{\text{ei}(b)}{\text{ei}'(b)}-x}{\op{ei}\left(x\right)}\right)}}\right)\frac{b-\frac{\text{ei}(b)}{\text{ei}'(b)}}{\frac{1}{\sqrt{2\pi}a^{2}}+a} & \le\frac{\mathbb{E}\left[\hat{\check{Y}}\right]}{\mathbb{E}\left[\hat{\check{F}}\right]}\label{eq:9990-1-1}\\
		\forall N\ge T\ge500,\bar{\mu},\Sigma,a=\sqrt{2\log N}, & b=\sqrt{\sqrt{2\log T}^{2}-2\log\left(3\ 2^{-3/2}\sqrt{2\log T}^{3}\right)}\nonumber 
	\end{align}
	
	\subsection*{Optimizing for $x$}
	
	Here we show that
	
	\begin{equation}
	\max_{x}\left(2\frac{b-\frac{\text{ei}(b)}{\text{ei}'(b)}-x}{\op{ei}\left(x\right)}\right)\le-\frac{2}{\text{ei}'(b)}.
	\end{equation}
	We have 
	\begin{align}
		\max_{x}\left(2\frac{b-\frac{\text{ei}(b)}{\text{ei}'(b)}-x}{\op{ei}\left(x\right)}\right) & =\max_{\delta}\left(2\frac{b-\frac{\text{ei}(b)}{\text{ei}'(b)}-b-\delta}{\op{ei}\left(b+\delta\right)}\right)\\
		& =\max_{\delta}\left(2\frac{\text{ei}(b)+\delta\text{ei}'(b)}{-\text{ei}'(b)\op{ei}\left(b+\delta\right)}\right).
	\end{align}
	Since the ei function is convex, we have $\text{ei}(b)+\delta\text{ei}'(b)\leq\text{ei}(b+\delta)$.
	Using this and the fact that the denominator is positive (since ei
	is always positive and $\op{ei}'$ is always negative), we have 
	\begin{equation}
	\max_{x}\left(2\frac{b-\frac{\text{ei}(b)}{\text{ei}'(b)}-x}{\op{ei}\left(x\right)}\right)\le-\frac{2}{\text{ei}'(b)}.
	\end{equation}
	Inserting this result into \ref{eq:9990-1-1}, we obtain 
	\begin{align}
		\left(1-e^{\frac{T\text{ei}'(b)}{2}}\right)\frac{b-\frac{\text{ei}(b)}{\text{ei}'(b)}}{\frac{1}{\sqrt{2\pi}a^{2}}+a} & \le\frac{\mathbb{E}\left[\hat{\check{Y}}\right]}{\mathbb{E}\left[\hat{\check{F}}\right]}\label{eq:9990-1-1-1}\\
		\forall N\ge T\ge500,\mu,\Sigma,a=\sqrt{2\log N}, & b=\sqrt{\sqrt{2\log T}^{2}-2\log\left(3\ 2^{-3/2}\sqrt{2\log T}^{3}\right)}.\nonumber 
	\end{align}
	
	\subsection*{Simplifying the bound further}
	
	Now, all that is left to do is to simplify this bound a bit further.
	
	\subsubsection*{Bounding the left factor}
	
	First of all, we show that the left factor satisfies 
	\begin{equation}
	1-e^{\frac{1}{2}T\text{ei}'(b)}\geq1-T^{-\frac{1}{2\sqrt{\pi}}}.\label{eq:left_factor_bound}
	\end{equation}
	Using \ref{def:ei} and the lower bound from \ref{lemma:bounds_on_normal_ccdf},
	we obtain 
	\begin{align}
		1-e^{\frac{1}{2}T\text{ei}'(b)} & \ge1-e^{-\frac{\left(b^{2}-1\right)e^{-\frac{b^{2}}{2}}T}{2\sqrt{2\pi}b^{3}}}\\
		& \ge1-e^{-\frac{e^{-\frac{b^{2}}{2}}T}{3\sqrt{2\pi}b}}
	\end{align}
	where the second inequality is easily seen to hold true, since we
	have $b\ge2.17$. Inserting \ref{eq:def_b}, and bounding further
	we obtain 
	\begin{align}
		1-e^{-\frac{e^{-\frac{b^{2}}{2}}T}{3\sqrt{2\pi}b}} & =1-e^{-\frac{\log^{\frac{3}{2}}(T)}{2\sqrt{\pi}\sqrt{\log(T)-\log\left(3\log^{\frac{3}{2}}(T)\right)}}}\\
		& \ge1-e^{-\frac{\log^{\frac{3}{2}}(T)}{2\sqrt{\pi}\sqrt{\log(T)}}}\\
		& =1-T^{-\frac{1}{2\sqrt{\pi}}}
	\end{align}
	from which \ref{eq:left_factor_bound} follows.
	
	\subsubsection*{Bounding the right factor}
	
	Now we will show that the right factor from \ref{eq:9990-1-1-1} satisfies
	\begin{equation}
	\frac{b-\frac{\text{ei}(b)}{\text{ei}'(b)}}{\frac{1}{\sqrt{2\pi}a^{2}}+a}\geq\sqrt{\frac{\log(T)-\log\left(3\log^{\frac{3}{2}}(T)\right)}{\log(N)}}.\label{eq:right_factor_bound}
	\end{equation}
	Using \ref{def:ei} and the upper bound from \ref{lemma:bounds_on_normal_ccdf},
	we obtain 
	\begin{equation}
	\frac{b-\frac{\text{ei}(b)}{\text{ei}'(b)}}{\frac{1}{\sqrt{2\pi}a^{2}}+a}\ge\frac{2\sqrt{\pi}a^{2}b^{5}}{\left(2\sqrt{\pi}a^{3}+\sqrt{2}\right)\left(b^{4}-b^{2}+3\right)}.\label{eq:right_factor_bound-1}
	\end{equation}
	Now, to lower bound this further, we show that 
	\begin{equation}
	\left(2\sqrt{\pi}a^{3}+\sqrt{2}\right)\left(b^{4}-b^{2}+3\right)\le\left(2\sqrt{\pi}a^{3}\right)b^{4}.\label{eq:denominator_boundssfdf}
	\end{equation}
	Equivalently, we can show that 
	\begin{equation}
	-2\sqrt{\pi}a^{3}b^{2}+6\sqrt{\pi}a^{3}+\sqrt{2}b^{4}-\sqrt{2}b^{2}+3\sqrt{2}\leq0.
	\end{equation}
	Since $b\ge2.17$, we have $3\sqrt{2}-\sqrt{2}b^{2}\leq0$, and hence
	
	\begin{equation}
	-2\sqrt{\pi}a^{3}b^{2}+6\sqrt{\pi}a^{3}+\sqrt{2}b^{4}\leq0\label{eq:sdfasdfasdf-2}
	\end{equation}
	is a sufficient condition. The derivative of the left-hand side 
	\begin{align}
		4\sqrt{2}b^{2}-4\sqrt{\pi}a^{3} & \le4\sqrt{2}a^{2}-4\sqrt{\pi}a^{3}\\
		& =(4\sqrt{2}-4\sqrt{\pi}a)a^{2}
	\end{align}
	is always negative (since $a\ge2.17$). Hence, it is sufficient to
	show that \ref{eq:sdfasdfasdf-2} holds for the minimal $b=2.17$,
	which can easily be verified to hold true for any $a\ge2.17$. Hence,
	we have shown that \ref{eq:denominator_boundssfdf} holds, and inserting
	it into \ref{eq:right_factor_bound-1}, we obtain 
	\begin{align}
		\frac{b-\frac{\text{ei}(b)}{\text{ei}'(b)}}{\frac{1}{\sqrt{2\pi}a^{2}}+a} & \ge\frac{2\sqrt{\pi}a^{2}b^{5}}{\left(2\sqrt{\pi}a^{3}\right)b^{4}}\\
		& =\frac{b}{a}\\
		& =\sqrt{\frac{\log(T)-\log\left(3\log^{\frac{3}{2}}(T)\right)}{\log(N)}}.
	\end{align}
	
	\subsection*{Final bound}
	
	Finally, inserting \ref{eq:left_factor_bound} and \ref{eq:right_factor_bound}
	into \ref{eq:9990-1-1-1}, we obtain 
	\begin{align}
		\left(1-T^{-\frac{1}{2\sqrt{\pi}}}\right)\sqrt{\frac{\log(T)-\log\left(3\log^{\frac{3}{2}}(T)\right)}{\log(N)}} & \le\frac{\mathbb{E}\left[\hat{\check{Y}}\right]}{\mathbb{E}\left[\hat{\check{F}}\right]}\quad\forall N\ge T\ge500,\mu,\Sigma.\label{eq:9990-1-1-1-1}
	\end{align}
	
	This implies straightforwardly the bound on the regret 
	\begin{equation}
	\frac{\mathbb{E}\left[\hat{\check{F}}\right]-\mathbb{E}\left[\hat{\check{Y}}\right]}{\mathbb{E}\left[\hat{\check{F}}\right]}\le1-\left(1-T^{-\frac{1}{2\sqrt{\pi}}}\right)\sqrt{\frac{\log(T)-\log\left(3\log^{\frac{3}{2}}(T)\right)}{\log(N)}}\quad\forall N\ge T\ge500,\mu,\Sigma.
	\end{equation}
	
\end{proof}

\section{Proof of \ref{theorem:main_standard}}

\label{sec:proof_standard} 

The proof of \ref{theorem:main_standard} is very similar to the one
of \ref{theorem:main} in \ref{sec:proof}. Here we discuss the parts
that are different.

 \begin{lemma} \label{lemma:bo_main_inequality_standard}

For any problem of the type defined in \ref{sec:po_problem_statement},
we have for any $\alpha,\beta>0$ 
\begin{align}
\beta\mathbb{E}\left[\hat{\check{F}}\right]-\mathbb{E}\left[\hat{Y}_{t}-\check{F}\right] & \le\alpha\left(\mathbb{E}\left[\hat{Y}_{t+1}-\hat{Y}_{t}\right]\right)\quad\forall t\in\{1:T-1\}\label{eq:bo_main_inequality_standard}\\
 & \Downarrow\nonumber \\
(1-e^{-\frac{T-1}{\alpha}})\beta\mathbb{E}\left[\hat{\check{F}}\right] & \le\mathbb{E}\left[\hat{Y}_{T}-\check{F}\right].
\end{align}
i.e. the first inequality implies the second inequality. \end{lemma} 

\begin{proof}

This result is obtained from \ref{lemma:bo_main_inequality} by replacing
$\hat{\check{Y}}_{t}$ with $\hat{Y}_{t}-\check{F}$. It is easy to
see that the proof of \ref{lemma:bo_main_inequality} goes through
with this change.

\end{proof}

Hence, if for some $\alpha,\beta>0$ we can show that \ref{eq:bo_main_inequality_standard}
holds, \ref{lemma:bo_main_inequality_standard} yields a lower bound
on the expected utility.

\begin{lemma} \label{lemma:extremization_main_inequality_standard}For
any problem of the type defined in \ref{sec:po_problem_statement}
we have for any $\alpha,\beta>0$ 
\begin{align}
\mathbb{E}\left[\beta\hat{\check{F}}-(\hat{Y}_{t}-\check{F})|m_{t},c_{t},\hat{y}_{t}\right] & \le\alpha\mathbb{E}\left[\hat{Y}_{t+1}-\hat{Y}_{t}|m_{t},c_{t},\hat{y}_{t}\right]\label{eq:extremization_main_inequality_standard}\\
 & \quad\quad\quad\quad\forall t\in\{1:T-1\},m_{t}\in\mathbb{R}^{N},c_{t}\in\mathbb{S}_{+}^{N},\hat{y}_{t}\in\mathbb{R}\nonumber \\
 & \Downarrow\nonumber \\
(1-e^{-\frac{T-1}{\alpha}})\beta\mathbb{E}\left[\hat{\check{F}}\right] & \le\mathbb{E}\left[\hat{Y}-\check{F}\right].
\end{align}
\end{lemma} 

\begin{proof}

The proof follows the same logic as the one from \ref{lemma:extremization_main_inequality}.

\end{proof}

\begin{theorem}\label{theorem:bound_alpha_standard} For any instance
$(N,T,\mu,\Sigma)$ of the problem defined in \ref{sec:po_problem_statement},
if we follow either the EI (\ref{def:ei_extremization_standard})
or the UCB (\ref{def:ucb_extremization_standard}) strategy, we have
for any $\alpha>0$ and any $0<\beta\le1-\frac{1}{\sqrt{2\pi}\left(2\log N\right)^{3/2}}$
\begin{align}
\max_{x}\left(2\frac{\beta\left(\sqrt{2\log N}+\frac{1}{2\log N\sqrt{2\pi}}\right)-x}{\op{ei}\left(x\right)}\right) & \le\alpha\label{eq:bound_alpha_condition_standard}\\
 & \Downarrow\nonumber \\
(1-e^{-\frac{T-1}{\alpha}})\beta\mathbb{E}\left[\hat{\check{F}}\right] & \le\mathbb{E}\left[\hat{Y}-\check{F}\right]\label{eq:bound_alpha_implication_standard}
\end{align}
i.e. the first line implies the second. \end{theorem} % Anteprima del sorgente dal paragrafo 78 al 81
\begin{proof}

According to \ref{lemma:extremization_main_inequality_standard},
we have 
\begin{align}
 & \ref{eq:bound_alpha_implication_standard}\nonumber \\
 & \Uparrow\nonumber \\
 & \frac{\mathbb{E}\left[\beta\hat{\check{F}}-(\hat{Y}_{t}-\check{F})|m_{t},c_{t},\hat{y}_{t}\right]}{\mathbb{E}\left[\hat{Y}_{t+1}-\hat{Y}_{t}|m_{t},c_{t},\hat{y}_{t}\right]}\le\alpha\nonumber \\
 & \quad\quad\forall t\in\{1:T-1\},m_{t}\in\mathbb{R}^{N},c_{t}\in\mathbb{S}_{+}^{N},\hat{y}_{t}\in\mathbb{R}.\label{eq:regretChain_standard_1}
\end{align}

Rearranging terms and plugging in the known variable $\hat{y}_t$ we have 
\begin{align}
 & \ref{eq:regretChain_standard_1}\nonumber \\
 & \Updownarrow\\
 & \frac{\beta\mathbb{E}\left[\hat{F}-\hat{y}_{t}|m_{t},c_{t}\right]-(1-\beta)\mathbb{E}\left[\hat{y}_{t}-\check{F}|m_{t},c_{t}\right]}{\mathbb{E}\left[\hat{Y}_{t+1}-\hat{y}_{t}|m_{t},c_{t},\hat{y}_{t}\right]}\le\alpha\nonumber \\
 & \quad\quad\forall t\in\{1:T-1\},m_{t}\in\mathbb{R}^{N},c_{t}\in\mathbb{S}_{+}^{N},\hat{y}_{t}\in\mathbb{R}.\label{eq:regretChain_standard_2}
\end{align}
Since the expected minimum function value $\check{F}$ is no larger
than the smallest mean $\check{m}_{t}$, we have

\begin{align}
 & \ref{eq:regretChain_standard_2}\nonumber \\
 & \Uparrow\\
 & \frac{\beta\mathbb{E}\left[\hat{F}-\hat{y}_{t}|m_{t},c_{t}\right]-(1-\beta)(\hat{y}_{t}-\check{m}_{t})}{\mathbb{E}\left[\hat{Y}_{t+1}-\hat{y}_{t}|m_{t},c_{t},\hat{y}_{t}\right]}\le\alpha\nonumber \\
 & \quad\quad\forall t\in\{1:T-1\},m_{t}\in\mathbb{R}^{N},c_{t}\in\mathbb{S}_{+}^{N},\hat{y}_{t}\in\mathbb{R}.\label{eq:regretChain_standard_3}
\end{align}
Since these terms are invariant to a common shift in all variables,
we can assume $\check{m}_{t}=0$ without loss of generality, and hence 

\begin{align}
 & \ref{eq:regretChain_standard_3}\nonumber \\
 & \Updownarrow\\
 & \frac{\beta\mathbb{E}\left[\hat{F}-\hat{y}_{t}|m_{t},c_{t}\right]-(1-\beta)\hat{y}_{t}}{\mathbb{E}\left[\hat{Y}_{t+1}-\hat{y}_{t}|m_{t},c_{t},\hat{y}_{t}\right]}\le\alpha\nonumber \\
 & \quad\quad\forall t\in\{1:T-1\},m_{t}\in\mathbb{R}_{\ge0}^{N},c_{t}\in\mathbb{S}_{+}^{N},\hat{y}_{t}\in\mathbb{R}_{\ge0}.\label{eq:regretChain_standard_4}
\end{align}
We have 
\begin{align}
\mathbb{E}\left[\hat{F}-\hat{y}_{t}|m_{t},c_{t}\right] & \le\mathbb{E}\left[\max\{\hat{F}-\hat{y}_{t},0\}|m_{t},c_{t}\right]\\
 & =\mathbb{E}\left[\max\{\hat{F},0\}|m_{t}-\hat{y}_{t},c_{t}\right]\\
 & =\text{mei}(0|m_{t}-\hat{y}_{t},c_{t})
\end{align}
with mei as defined in \ref{def:nei}. In addition, we have 
\begin{align}
\mathbb{E}\left[\hat{Y}_{t+1}-\hat{y}_{t}|m_{t},c_{t},\hat{y}_{t}\right] & =\mathbb{E}\left[\max\left\{ \hat{Y}_{t+1}-\hat{y}_{t},0\right\} |m_{t},c_{t},\hat{y}_{t}\right]\\
 & \ge\text{ei}\left(\hat{y}_{t}|m_{t}^{n_{ucb}},\sqrt{c_{t}^{n_{ucb}n_{ucb}}}\right)
\end{align}
with ei as defined in \ref{def:ei}. The equality follows from the fact that we know that $\hat{Y}_{t+1}\ge \hat{y}_t$. The inequality follows due to a similar argument as the one in the proof of \ref{lemma:bound_rhs}.

Substituting these terms and dropping the time index, we have 
\begin{align}
 & \ref{eq:regretChain_standard_4}\nonumber \\
 & \Uparrow\nonumber \\
 & \frac{\beta\op{mei}(0|m-\hat{y},c)-(1-\beta)\hat{y}}{\op{ei}\left(\hat{y}\bigg|m^{n_{ucb}},\sqrt{c^{n_{ucb}n_{ucb}}}\right)}\le\alpha\nonumber \\
 & \quad\quad\forall m\in\mathbb{R}^{N}_{\ge 0},c\in\mathbb{S}_{+}^{N},\hat{y}\in\mathbb{R}_{\ge 0}\label{eq:regretChain-1-1-1-1}\\
 & \quad\quad\quad\quad n_{ucb}=\op{argmax}_{n\in[N]}\max\left(m^{n}+\sqrt{c^{nn}2\log N}\right).
\end{align}
Note that this condition is implied by \ref{eq:regretChain2}, hence
the rest of the proof is identical to the one from \ref{theorem:bound_alpha}.
\end{proof}

Finally, using the previous results, we can obtain the desired bound
on the regret (\ref{theorem:main_standard}), which we restate here
for convenience:

\maintheoremstandard*

\begin{proof}
The proof is identical to the one from \ref{theorem:main}. We simply
substitute $\hat{\check{Y}}$ with $\hat{Y}-\check{F}$ and the proof
goes through unchanged otherwise.
\end{proof}

\section{Proof of \ref{lemma:lowerregretbound}}\label{app:lowerregretbound}
For convenience, we restate \ref{lemma:lowerregretbound} before the proof:
\lowerregretbound*
\begin{proof}Since the bandits are i.i.d., an optimal strategy
	is to select arms uniformly at random (without replacement),
	which means that at each step we observe an i.i.d. sample. It is known that for
	i.i.d. standard normal random variables $X_{1},...,X_{K}$, we have
	asymptotically 
	\begin{equation}
	\mathbb{E}[\max\{X_{1},...,X_{K}\}]\sim\sqrt{2\log K}\quad(K\to\infty),
	\end{equation}
	see e.g. \cite{massart2007concentration} page 66. It follows that
	$\forall\epsilon>0~\exists K:\forall N\ge T\ge K:$
	\begin{equation}
	\frac{\mathbb{E}[\hat{Y}]}{\mathbb{E}[\hat{F}]}\le\frac{\sqrt{2\log T}}{\sqrt{2\log N}}+\epsilon
	\end{equation}
	from which the desired result follows straightforwardly.
\end{proof}

\section{Proof of \ref{lemma:priordependentpolicies}}\label{app:priordependentpolicies}
For convenience, we restate \ref{lemma:priordependentpolicies} before the proof:
\priordependentpolicies*
\begin{proof}
	Suppose we have an optimization policy, and the prior mean $\mu$
	and covariance $\Sigma$ are both zero, i.e. the GP is zero everywhere.
	This will induce a distribution over the actions $A_{1:T}$ taken by the policy.
	Since there are $N$ possible actions and the policy is allowed to
	pick $T$ of them, there is at least one action $K$ which has a probability
	of no more than $T/N$ of being chosen. Now suppose that we set the
	prior covariance $\Sigma_{KK}$ to a nonzero value, while maintaining
	everything else zero. Note that this will not change the distribution
	over the actions taken by the policy unless it happens to pick
	$K$, hence the probability of action $K$ being selected does not
	change. Since the observed maximum $\hat{Y}$ is $\hat{F}$ if action $K$ is selected
	by the policy and zero otherwise, we have
	\begin{align}
	\mathbb{E}[\hat{Y}] & \le\mathbb{E}[\hat{F}]\frac{T}{N}
	\end{align}
	from which \ref{blablalladfa} follows straightforwardly.
\end{proof}

\section{Proof of \ref{thm:continuous}}\label{app:continuous}
For convenience, we restate \ref{thm:continuous} before the proof:
\continuoustheorem*
\begin{proof}
	The idea here is to pre-select a set of $N$ points at locations $X_{1:N}$
	on a grid and then sub-select points from this set during runtime using EI2 (\ref{def:ei_extremization}) or UCB2 (\ref{def:ucb_extremization}).
	The regret of this strategy can be bounded by 
	\begin{equation}
	\mathbb{E}\left[\sup_{a\in\mathcal{A}}G(a)-\hat{Y}\right]\le\mathbb{E}\left[\sup_{a\in\mathcal{A}}G(a)-\max_{i\in[N]}G_{X_{i}}\right]+\mathbb{E}\left[\max_{i\in[N]}G_{X_{i}}-\hat{Y}\right].\label{eq:combined}
	\end{equation}
	A bound on the first term is given by the main result in \citet{Grunewalder2010-uu}:
	\begin{align}
	\mathbb{E}\left[\sup_{a\in\mathcal{A}}G(a)-\max_{i\in[N]}G_{X_{i}}\right] & \le\sqrt{\frac{2L_{k}}{\left\lfloor N^{1/D}\right\rfloor }}\left(2\sqrt{\log\left(2N\right)}+15\sqrt{D}\right).\label{eq:mygrunewalder}
	\end{align}
	A bound on the second term can be derived straightforwardly from \ref{coro:maximization}:
	\begin{align}
	\mathbb{E}\left[\max_{i\in[N]}G_{A_{i}}-\hat{Y}\right] & \le\left(1-\left(1-T^{-\frac{1}{2\sqrt{\pi}}}\right)\sqrt{\frac{\log\left(\frac{T}{3\log^{\frac{3}{2}}(T)}\right)}{\log(N)}}\right)\max_{i\in[N]}G_{A_{i}}\\
	 & \le\sqrt{2}\sigma\left(\sqrt{\log(N)}-\left(1-T^{-\frac{1}{2\sqrt{\pi}}}\right)\sqrt{\log\left(\frac{T}{3\log^{\frac{3}{2}}(T)}\right)}\right).
	\end{align}
	where we have used the inequality $\max_{i\in[N]}G_{A_{i}}\le\sigma\sqrt{2\log(N)}$
	(see \cite{massart2007concentration}, Lemma 2.3). 
	
	Choosing $N=\left\lceil \frac{L_{k}}{\log(L_{k})}T^{1/D}\right\rceil ^{D}$
	we obtain 
	
	\begin{align}
	\mathbb{E}\left[\sup_{a\in\mathcal{A}}G(a)-\max_{i\in[N]}G_{X_{i}}\right]\le & \sqrt{\frac{2L_{k}}{\left\lceil \frac{L_{k}}{\log(L_{k})}T^{1/D}\right\rceil }}\left(2\sqrt{\log\left(2\left\lceil \frac{L_{k}}{\log(L_{k})}T^{1/D}\right\rceil ^{D}\right)}+15\sqrt{D}\right)\\
	\le & \sqrt{\frac{2\log(L_{k})}{T^{1/D}}}\left(2\sqrt{\log\left(2\left\lceil \frac{L_{k}}{\log(L_{k})}T^{1/D}\right\rceil ^{D}\right)}+15\sqrt{D}\right)
	\end{align}
	
	and 
	\begin{align}
	\mathbb{E}\left[\max_{i\in[N]}G_{A_{i}}-\hat{Y}\right] & \le\sqrt{2}\sigma\left(\sqrt{D\log\left(\left\lceil \frac{L_{k}}{\log(L_{k})}T^{1/D}\right\rceil \right)}-\left(1-T^{-\frac{1}{2\sqrt{\pi}}}\right)\sqrt{\log\left(\frac{T}{3\log^{\frac{3}{2}}(T)}\right)}\right).
	\end{align}
	Substituting these terms in \ref{eq:combined} yields \ref{eq:continuous_bound}.
	
	So far, we have only shown that \ref{eq:continuous_bound} holds when
	preselecting a grid as proposed in \citet{Grunewalder2010-uu} and
	then restricting our optimization policies to this preselected domain.
	However, it is easy to show that the result also holds when allowing
	EI2 (\ref{def:ei_extremization}) or the UCB2 (\ref{def:ucb_extremization})
	to select from the entire domain $\mathcal{A}$. The proof of \Cref{coro:maximization}
	is based on bounding the expected increment of the observed maximum
	at each time step \ref{eq:bo_main_inequality}. It is clear that by
	allowing the policy to select from a larger set of points, the expected
	increment cannot be smaller. 

	\subsubsection*{Convergence}

% Preview source code from paragraph 21 to 22

In the following, we show that the bound converges to 0 as $T\to\infty$.
Clearly, the first term in \ref{eq:continuous_bound} converges to
zero. The second term can be written (without the factor $\sqrt{2}\sigma$,
as it is irrelevant) as

\begin{equation}
\sqrt{D\log\left(\left\lceil \frac{L_{k}}{\log(L_{k})}T^{1/D}\right\rceil \right)}-\sqrt{\log\left(\frac{T}{3\log^{\frac{3}{2}}(T)}\right)}+T^{-\frac{1}{2\sqrt{\pi}}}\sqrt{\log\left(\frac{T}{3\log^{\frac{3}{2}}(T)}\right)}.
\end{equation}
Clearly, the last of these terms converges to zero. For the other two terms
we have 
\begin{align}
 & \sqrt{D\log\left(\left\lceil \frac{L_{k}}{\log(L_{k})}T^{1/D}\right\rceil \right)}-\sqrt{\log\left(\frac{T}{3\log^{\frac{3}{2}}(T)}\right)}\\
 & \le\sqrt{D\log\left(\frac{L_{k}}{\log(L_{k})}T^{1/D}+1\right)}-\sqrt{\log\left(\frac{T}{3\log^{\frac{3}{2}}(T)}\right)}\\
 & =\frac{D\log\left(\frac{L_{k}}{\log(L_{k})}T^{1/D}+1\right)-\log\left(\frac{T}{3\log^{\frac{3}{2}}(T)}\right)}{\sqrt{D\log\left(\frac{L_{k}}{\log(L_{k})}T^{1/D}+1\right)}+\sqrt{\log\left(\frac{T}{3\log^{\frac{3}{2}}(T)}\right)}}\\
 & =\frac{D\log\left(\frac{L_{k}}{\log(L_{k})}+T^{-1/D}\right)+\log\left(3\right)+\frac{3}{2}\log\left(\log(T)\right)}{\sqrt{D\log\left(\frac{L_{k}}{\log(L_{k})}T^{1/D}+1\right)}+\sqrt{\log\left(\frac{T}{3\log^{\frac{3}{2}}(T)}\right)}}.
\end{align}
The numerator grows with $\log\log T$, while the denominator grows
faster, with $\sqrt{\log T}$, which means that these terms also converge
to 0.

	\end{proof}

\section{Some Properties of Expected Improvement and Related Functions}

% Anteprima del sorgente dal paragrafo 107 al 117
Here we prove some properties of the expected improvement $\op{ei}$  (\ref{def:ei})
and related functions, such that we can use them in the rest of the
proof.

\begin{lemma}[Convexity of expected improvement]\label{lemma:ei_is_convex}
	The standard expected improvement function $\op{ei}(x)$ (\ref{def:ei}) is convex
	on $\mathbb{R}$.
	
\end{lemma}

\begin{proof}
	
	This is easy to see since the second derivative 
	\begin{equation}
	\frac{d^{2}\op{ei}(x)}{dx^{2}}=\frac{e^{-\frac{x^{2}}{2}}}{\sqrt{2\pi}}
	\end{equation}
	is positive everywhere.
	
\end{proof}

\begin{lemma}[Bounds on normal CCDF]\label{lemma:bounds_on_normal_ccdf}
	We have the following bounds on the complementary cumulative density
	function (CCDF) of the standard normal distribution 
	\begin{equation}
	\left(\tau^{-1}-\tau^{-3}\right)\mathcal{N}(\tau)\le\Phi^{c}(\tau)\le\left(\tau^{-1}-\tau^{-3}+3\tau^{-5}\right)\mathcal{N}(\tau)\quad\forall\tau>0.
	\end{equation}
	The normal CCDF is defined as 
	\begin{equation}
	\Phi^{c}(\tau):=\int_{\tau}^{\infty}\mathcal{N}(x)dx.
	\end{equation}
	
\end{lemma}

\begin{proof}
	
	Using integration by parts, we can write for any $\tau>0$ 
	\begin{align}
		\Phi^{c}(\tau) & =\int_{\tau}^{\infty}\mathcal{N}(x)dx\\
		& =\tau^{-1}\mathcal{N}(\tau)-\int_{\tau}^{\infty}x^{-2}\mathcal{N}(x)dx\\
		& =\left(\tau^{-1}-\tau^{-3}\right)\mathcal{N}(\tau)+3\int_{\tau}^{\infty}x^{-4}\mathcal{N}(x)dx\\
		& =\left(\tau^{-1}-\tau^{-3}+3\tau^{-5}\right)\mathcal{N}(\tau)-15\int_{\tau}^{\infty}x^{-6}\mathcal{N}(x)dx.
	\end{align}
	From this the bounds straightforwardly follow
	
\end{proof}

\begin{lemma}[Bounds on the expected improvement]\label{lemma:bound_ei}
	We have the following bound for the expected improvement (\ref{def:ei})
	\begin{align}
		\left(\tau^{-2}-3\tau^{-4}\right)\mathcal{N}(\tau)\le\op{ei}(\tau)\le\tau^{-2}\mathcal{N}(\tau)\quad\forall\tau>0.
	\end{align}
\end{lemma}

\begin{proof} From \ref{def:ei} we have that 
	\begin{align}
		\op{ei}(\tau)=\mathcal{N}(\tau)-\tau\Phi^{c}(\tau).
	\end{align}
	The desired result follows straightforwardly using the bounds from
	\ref{lemma:bounds_on_normal_ccdf}.
	
\end{proof}

\begin{lemma}[Upper bound on the multivariate expected improvement]\label{lemma:bound_nei}
	% Anteprima del sorgente per il paragrafo 102
	For a family of jointly Gaussian distributed RVs $(F_{n})_{n\in[N]}$
	with mean $m\in\mathbb{R}^{N}$ and covariance $c\in\mathbb{S}_{+}^{N}$
	and a threshold $\tau\in\mathbb{R}$, we can bound the multivariate
	expected improvement (\ref{def:nei}) as follows 
	\begin{align}
		\op{mei}(\tau|m,c) & =\mathbb{E}\left[\max\left\{ \max_{n\in[N]}F_{n}-\tau,0\right\} \right]\\
		& \le\max\left\{ \max_{n\in[N]}(m_{n}-\tau+\sqrt{c_{nn}2\log N}),0\right\} +\frac{\max_{n\in[N]}\sqrt{c_{nn}}}{2\sqrt{2\pi}\log\left(N\right)}.
	\end{align}
	
\end{lemma}

% Anteprima del sorgente per il paragrafo 97
\begin{proof} Defining $Z:=\max_{n\in[N]}F_{n}$, we can write 
	\begin{align}
		\op{mei}(0|m,c) & =\mathbb{E}\left[\max\left\{ Z,0\right\} \right]\\
		& =\int_{-\infty}^{\infty}\max\{z,0\}p(z)dz\\
		& =\int_{0}^{\infty}zp(z)dz\\
		& \stackrel{{\scriptstyle b\ge0}}{=}\int_{0}^{b}zp(z)dz+\int_{b}^{\infty}zp(z)dz\label{eq:introduced_condition_on_b}\\
		& =\int_{0}^{b}zp(z)dz+\int_{0}^{\infty}\int_{b}^{\infty}[x\le z]p(z)dzdx\\
		& =\int_{0}^{b}zp(z)dz+\int_{0}^{\infty}P(Z\ge\max(x,b))dx\\
		& =\int_{0}^{b}zp(z)dz+bP(Z\ge b)+\int_{b}^{\infty}P(Z\ge x)dx.\label{eq:mei_bound_chain_1}
	\end{align}
	Since $\int_{0}^{b}zp(z)dz\le\int_{0}^{b}bp(z)dz$ for any $b\ge0$,
	we can bound this as 
	\begin{align}
		\ref{eq:mei_bound_chain_1} & \le bP(0\le Z\le b)+bP(Z\ge b)+\int_{b}^{\infty}P(Z\ge x)dx\\
		& =bP(0\le Z)+\int_{b}^{\infty}P(Z\ge x)dx\\
		& \le b+\int_{b}^{\infty}P(Z\ge x)dx\\
		& =b+\int_{b}^{\infty}P(\lor_{n\in[N]}(F_{n}\ge x))dx\label{eq:mei_bound_chain_2}
	\end{align}
	where we have used the fact that the event $Z\ge x$ is identical
	to the event $\lor_{n\in[N]}(F_{n}\ge x)$. Using the union bound
	we can now write 
	\begin{align}
		\ref{eq:mei_bound_chain_2} & \le b+\sum_{n\in[N]}\int_{b}^{\infty}P(F_{n}\ge x)dx\\
		& =b+\sum_{n\in[N]}\int_{b}^{\infty}\int_{-\infty}^{\infty}[f_{n}\ge x]p(f_{n})df_{n}dx\\
		& =b+\sum_{n\in[N]}\int_{-\infty}^{\infty}\max\{f_{n}-b,0\}p(f_{n})df_{n}.\label{eq:mei_bound_chain_3}
	\end{align}
	Noticing that the summands in the above term match the definition
	of expected improvement (\ref{def:ei}), we can write 
	\begin{align}
		\ref{eq:mei_bound_chain_3} & =b+\sum_{n\in[N]}\op{ei}\left(b|m_{n},\sqrt{c_{nn}}\right)\\
		& \le b+N\max_{n\in[N]}\op{ei}\left(b|m_{n},\sqrt{c_{nn}}\right)\\
		& =b+N\max_{n\in[N]}\sqrt{c_{nn}}\op{ei}\left(\frac{b-m_{n}}{\sqrt{c_{nn}}}\right)\label{eq:mei_bound_chain_4}
	\end{align}
	where we have used \ref{eq:ei_normalization} in the last line. Using
	\ref{lemma:bound_ei} we obtain for any $b>\max_{n}m_{n}$ 
	\begin{align}
		\ref{eq:mei_bound_chain_4} & \le b+N\max_{n\in[N]}\left(\sqrt{c_{nn}}\left(\frac{b-m_{n}}{\sqrt{c_{nn}}}\right)^{-2}\mathcal{N}\left(\frac{b-m_{n}}{\sqrt{c_{nn}}}\right)\right).\label{eq:mei_bound_chain_6}
	\end{align}
	This inequality hence holds for any $b$ which satisfies $b>\max_{n}m_{n}$
	and $b\ge0$ (from \ref{eq:introduced_condition_on_b}). We pick the
	following value which clearly satisfies these conditions 
	\begin{align}
		b & =\max\left\{ \max_{n\in[N]}\left(m_{n}+\sqrt{c_{nn}2\log N}\right),0\right\} \label{eq:bo_b_hat}
	\end{align}
	from which it follows that 
	\begin{align}
		\frac{b-m_{n}}{\sqrt{c_{nn}}} & \ge\sqrt{2\log N}\quad\forall n\in[N].
	\end{align}
	Given this fact it is easy to see that 
	\begin{align}
		\ref{eq:mei_bound_chain_6} & \le b+N\max_{n\in[N]}\left(\sqrt{c_{nn}}\sqrt{2\log N}^{-2}\mathcal{N}\left(\sqrt{2\log N}\right)\right)\\
		& =b+\frac{1}{2\sqrt{2\pi}\log N}\max_{n\in[N]}\sqrt{c_{nn}}\\
		& =\max\left\{ \max_{n\in[N]}\left(m_{n}+\sqrt{c_{nn}2\log N}\right),0\right\} +\frac{\max_{n\in[N]}\sqrt{c_{nn}}}{2\sqrt{2\pi}\log N}.
	\end{align}
	It is easy to see that $\op{mei}(\tau|m,c)=\op{mei}(0|m-\tau,c)$,
	which concludes the proof. \end{proof}

\end{document}